\documentclass[Afour,sageh,times]{sagej}

\usepackage{moreverb}
\usepackage{url}
\usepackage{textcomp}
\usepackage{stfloats}
\usepackage{verbatim}
\usepackage{paralist}

\usepackage{amsmath,amsfonts}
\usepackage{amsthm}
\usepackage{amssymb}
\usepackage{array}
\usepackage{mathtools}

\usepackage[ruled, vlined, linesnumbered]{algorithm2e}
\usepackage{graphicx}
\usepackage[font=small]{caption}
\usepackage[subrefformat=parens]{subcaption}
\usepackage[mathscr]{euscript}

\usepackage{multirow}

\usepackage{tikz}
\usetikzlibrary{automata, positioning, arrows, matrix}
\tikzset{initial text={}} 

\usepackage{xcolor}

\newcommand{\rev}[1]{\textcolor{black}{#1}}

\newcommand\BibTeX{{\rmfamily B\kern-.05em \textsc{i\kern-.025em b}\kern-.08em
T\kern-.1667em\lower.7ex\hbox{E}\kern-.125emX}}

\usepackage[colorlinks,bookmarksopen,bookmarksnumbered,citecolor=red,urlcolor=red]{hyperref}
\usepackage{cleveref}

\def\volumeyear{2025}

\newtheorem{theorem}{Theorem}

\newtheorem{proposition}{Proposition}

\newtheorem{lemma}{Lemma}
\newtheorem{definition}{Definition}
\newtheorem{remark}{Remark}
\newtheorem{problem}{Problem}
\newtheorem{myexam}{Example}

\DeclareMathOperator*{\argmin}{arg\,min}

\newcommand{\set}[1]{\{#1\}}

\newcommand\scr[1]{\mathcal{#1}}

\newcommand{\AP}{\Pi}
\newcommand{\ap}{\pi}

\newcommand{\PowerSetofAP}{2^\AP}
\newcommand{\trace}{\omega}
\newcommand{\traceSequence}{\trace_1 \trace_2 \ldots \trace_n}
\newcommand{\traces}{\Omega}
\newcommand{\word}{\trace}
\newcommand{\infword}{\trace}

\newcommand{\run}{z}
\newcommand{\runSequence}{z_0 z_1 \ldots z_n}
\newcommand{\plan}{\gamma}

\newcommand{\observation}{\rho_\plan}

\newcommand{\nextOp}{\mathcal{X}}

\newcommand{\globallyOp}{\mathcal{G}}
\renewcommand{\L}{\mathcal{L}}      

\newcommand{\cosafe}{\text{task}}   
\newcommand{\safe}{\text{safe}}     

\newcommand{\A}{\mathcal{A}}        
\newcommand{\Pd}{\mathcal{P}}       

\newcommand{\PP}{\mathbb{P}}        
\newcommand{\PA}{\A^\PP}            

\newcommand{\BLUE}{Q_\text{blue}}
\newcommand{\RED}{Q_\text{red}}

\newcommand{\water}{w}
\newcommand{\green}{g}

\newcommand{\carpet}{c}
\newcommand{\emptyAP}{e}
\newcommand{\fish}{f}
\newcommand{\ship}{s}
\newcommand{\target}[1]{$[#1]$}
\newcommand{\deltaA}{\delta_{\mathcal{A}}}

\newcommand{\deltaPA}{\delta_{\PP}}
\newcommand{\FinalPA}{F_{\PP}}
\newcommand{\PATuple}{(Q, \Sigma, q_0, \deltaA, F, \deltaPA, \FinalPA)}

\newcommand{\Sys}{\text{R}}
\newcommand{\Env}{\text{E}}
\newcommand{\SysState}{S_\Sys}
\newcommand{\EnvState}{S_\Env}
\newcommand{\SysAction}{A_\Sys}
\newcommand{\EnvAction}{A_\Env}

\newcommand{\varphiSafe}{\varphi_\safe}
\newcommand{\ASafe}{\A_\safe}

\newcommand{\QSafe}{Q^s}
\newcommand{\qInitSafe}{q^s_0}
\newcommand{\deltaSafe}{\delta^s}
\newcommand{\FinalSafe}{F^s}
\newcommand{\ASafeTuple}{(\QSafe, \Sigma, \qInitSafe, \deltaSafe, \FinalSafe)}

\newcommand{\play}{\mathscr{S}}
\newcommand{\playSequence}{s_0 s_1 \ldots s_n}

\DeclareMathOperator*{\PlaysOf}{Play}

\DeclareMathOperator*{\totalpayoff}{TP}

\newcommand{\strategy}{\tau}

\newcommand{\tpSet}{U}

\newcommand{\tpSetAfter}[1]{
    \tpSet_{#1}
}

\newcommand{\pp}{p}
\newcommand{\ppSet}{\mathscr{P}}

\newcommand{\Gstate}{s}
\newcommand{\GState}{S}
\newcommand{\GAction}{A}
\newcommand{\GInit}{s_0}
\newcommand{\GDelta}{\delta}
\newcommand{\GLabel}{L}
\newcommand{\GWeight}{W}
\newcommand{\DTGTuple}{(\GState, \GAction, \GInit, \GDelta, \AP, \GLabel, \GWeight)}

\newcommand{\augG}{{\bar{\G}}}
\newcommand{\augGState}{\bar{S}}
\newcommand{\augGAction}{A}
\newcommand{\augGDelta}{\bar{\delta}}
\newcommand{\augGInit}{\bar{s}_0}
\newcommand{\augGLabel}{\bar{L}}
\newcommand{\augDTGTuple}{(\augGState, \augGAction, \augGInit, \augGDelta, \AP, \augGLabel)}

\newcommand{\G}{G}        
\newcommand{\GPd}{\Pd^\G}
\newcommand{\GPdstate}{s^\Pd}
\newcommand{\GPdstateprime}{s^{\prime \Pd}}
\newcommand{\GPdState}{S^\Pd}
\newcommand{\GPdAction}{A}
\newcommand{\GPdInit}{s_0^\Pd}
\newcommand{\GPdTerm}{s_t^\Pd}
\newcommand{\GPdEdge}{E^\Pd}
\newcommand{\GPdWeight}{W^\Pd}
\newcommand{\GPdTuple}{(\GPdState, \GPdAction, \GPdInit, \GPdTerm, \GPdEdge, \GPdWeight)}
\newcommand{\GPdNumState}{|\GPdState|}
\newcommand{\GPdNumEdge}{|\GPdEdge|}
\newcommand{\GPdMaxStep}{\GPdNumState-1}
\newcommand{\GPdComplexity}{\mathcal{O}(\GPdNumState(\GPdNumState + \GPdNumEdge))}

\newcommand{\pathGPd}{\Lambda^{\Pd}}

\newcommand{\fpoperator}{F}

\begin{document}

\runninghead{Kandai, Renninger, and et al.}

\title{Learning specifications for reactive synthesis with safety constraints}

\author{
Kandai Watanabe\affilnum{1}, 
Nicholas Renninger\affilnum{2}, 
Sriram Sankaranarayanan\affilnum{3}
and Morteza Lahijanian\affilnum{3}
}

\affiliation{\affilnum{1}Google LLC\\
\affilnum{2}MITRE Corporation\\
\affilnum{3}University of Colorado Boulder}

\corrauth{Kandai Watanabe, Google LLC,
901 Cherry Avenue, San Bruno, CA}

\email{kandai.watanabe@colorado.edu}

\begin{abstract}
This paper presents a novel approach to \emph{learning from demonstration} that enables robots to autonomously execute complex tasks in dynamic environments. We model latent tasks as probabilistic formal languages and introduce a tailored reactive synthesis framework that balances robot costs with user task preferences. Our methodology focuses on safety-constrained learning and inferring formal task specifications as Probabilistic Deterministic Finite Automata (PDFA). We adapt existing ``evidence-driven state merging" algorithms and incorporate safety requirements throughout the learning process to ensure that the learned PDFA always complies with safety constraints. Furthermore, we introduce a multi-objective reactive synthesis algorithm that generates deterministic strategies that are guaranteed to satisfy the PDFA task while optimizing the trade-offs between user preferences and robot costs, resulting in a Pareto front of optimal solutions. Our approach models the interaction as a two-player game between the robot and the environment, accounting for dynamic changes. We present a computationally-tractable value iteration algorithm to generate the Pareto front and the corresponding deterministic strategies.
Comprehensive experimental results demonstrate the effectiveness of our algorithms across various robots and tasks, showing that the learned PDFA never includes unsafe behaviors and that synthesized strategies consistently achieve the task
while meeting both the robot cost and user-preference requirements.
\end{abstract}

\maketitle

\section{Introduction}



Technological advancements are enabling robots to operate with increasing autonomy in human-shared domains. Examples range from home assistive robots and assembly lines to deep-sea and planetary exploration. In these environments, robots must make decisions to achieve \emph{complex tasks} in diverse, dynamic conditions while adhering to strict \emph{safety} requirements. However, complex task specifications are often unavailable or too difficult for non-experts to provide. Instead, tasks can be demonstrated through human operation or past data. The robot must then infer the task objective and execute it autonomously. This process presents five challenges: 
(i) identifying a formalism for efficient and precise learning from demonstrations, 
(ii) ensuring the learned specification satisfies safety properties, 
(iii) capturing operator preferences or hidden costs, 
(iv) applying the specifications in new environments, and 
(v) ensuring task completion with reactivity. 
In this article, we aim to address these challenges by drawing on formal methods to develop a specification-learning scheme and a reactive strategy synthesis algorithm that effectively complement each other.

Consider, for instance, an underwater robot deployed for deep-sea scientific exploration, as depicted in \Cref{fig:introductory_example}. The scientists want the robot to investigate a shipwreck on the ocean floor, observe the behavior of a school of fish, and steer clear of coral reefs to prevent damage. Rather than requiring a roboticist and domain specialist to either remotely operate the robot or meticulously define the mission objectives and execution plan, the goal of this work is to enable the autonomous execution of this task by just exposing the robot to the data (demonstrations) of similar missions in previous deployments. From such data, the robot should be able to infer the robust task representation and generate the necessary strategy to accomplish the task even under dynamically changing environments. We call this problem \textit{Specification Learning from Demonstrations}, which can be regarded as a new form of \textit{Learning from Demonstration} (LfD) \cite{ravichandar2020recent}.

\begin{figure}
    \centering
    \includegraphics[width=\linewidth]{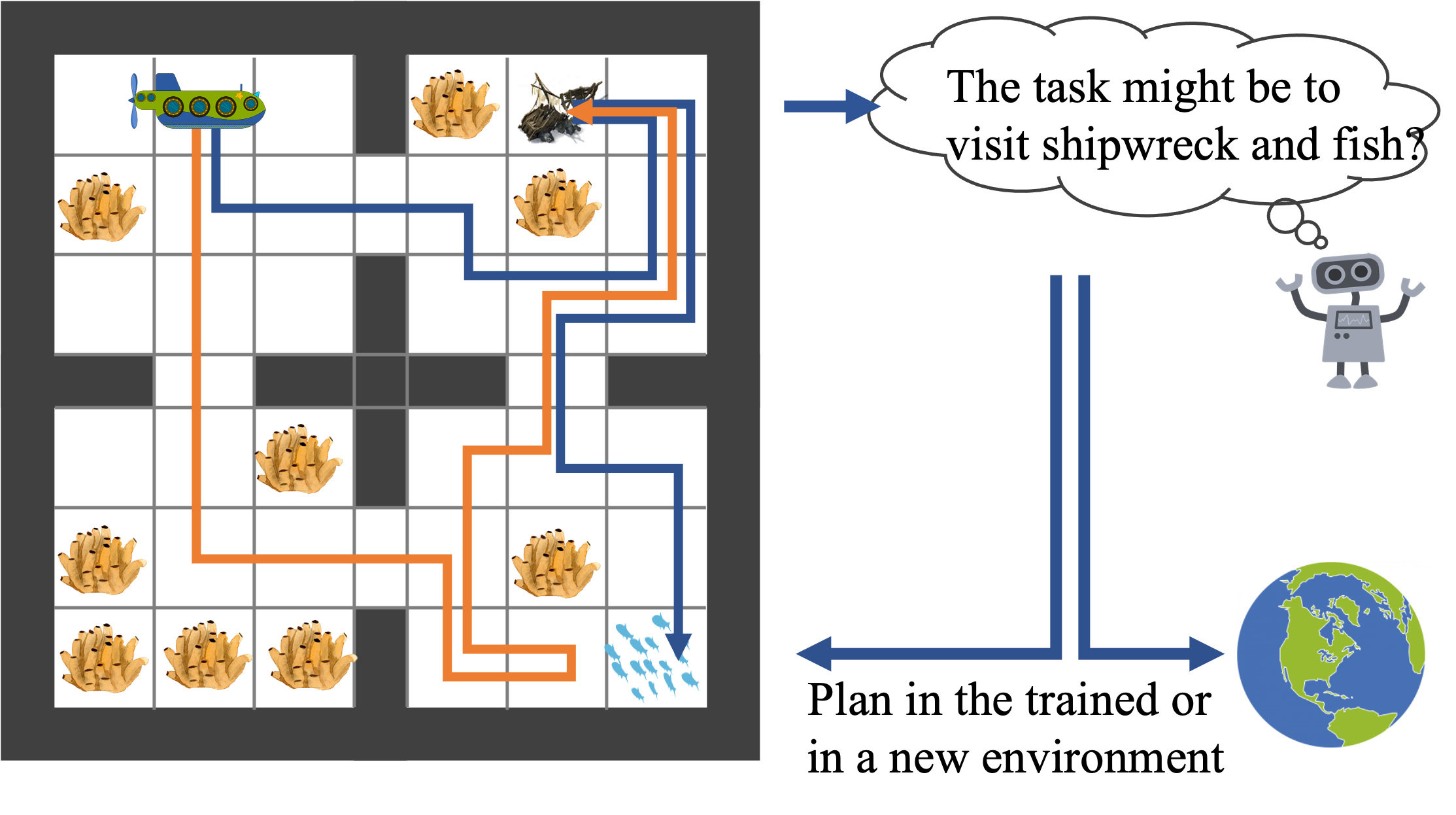}
    \caption{Schematic  of an autonomous deep-sea science mission. The task for the green underwater vehicle is to visit a school of fish (blue) and shipwreck (brown) while avoiding coral reefs (yellow). Our goal is to infer the underlying task from the demonstrations (the colored path) and synthesize a controller in the trained/new environment that achieves the learned task.}
    \label{fig:introductory_example}
\end{figure}


Most existing approaches to LfD and reactive planning focus on learning a reward structure or policy \cite{ravichandar2020recent, hussein2017imitation}. These methods typically learn a function specific to the environment and robot model used during training, making them fragile to changes in those models. 
In many real-world scenarios, however, demonstrations are performed in settings different from the execution environment.
Furthermore, these approaches are limited to Markovian tasks, where decisions at the current state depend only on the present and not on past events. But, complex tasks usually require maintaining a history of past events for successful completion.  For example, the robot task in \Cref{fig:introductory_example} involves visiting both the shipwreck and the school of fish in any order, making it non-Markovian. In such tasks, the robot must track its previous locations to decide the next one. 

In this work, we propose a novel approach to LfD by viewing tasks as probabilistic formal languages and introduce a reactive synthesis framework that optimally trades off robot's operational costs with user preferences on how the task should be completed.  We infer formal task specifications as probabilistic automata, drawing insights from the \emph{grammatical inference} (GI) domain \cite{de2010grammatical}.  
\rev{
Representing tasks as automata offers several advantages: (i) interpretability, (ii) symbolic reasoning capabilities, and (iii) access to well-studied algorithmic manipulation techniques.
}
We focus on \textit{safety-property constrained learning}, wherein the final learned specification must satisfy the safety properties. 
Distinctly, our proposed learning technique incorporates safety constraints throughout the learning process,
rather than applying them post hoc.
Specifically, we adapt existing 
\rev{Evidence-Driven State Merging (EDSM)
algorithms in GI to learn the task specification as a \textit{Probabilistic Deterministic Finite Automaton} (PDFA)~\cite{de2010grammatical}.
We integrate safety properties into the learning process, ensuring that all execution traces of the resultant PDFA satisfy the safety requirements.
}

Given the inferred PDFA, we present a reactive synthesis algorithm that generates deterministic strategies to accomplish the task while concurrently maximizing the demonstrator's preferences and minimizing the robot's operational costs. 
However, these objectives often conflict, resulting in a trade-off. 
\rev{
This leads to multiple optimal solutions, collectively known as the \emph{Pareto front}, where each solution offers a different balance between the task preference and cost minimization~\cite{chen2013stochastic}.
}
We propose a computationally efficient algorithm to compute the Pareto front and derive strategies for each Pareto point. This approach models the problem as a two-player game between the robot (system) and the environment, treating the environment as an adversarial agent to handle dynamic changes. Our experimental results, tested across various robots and tasks, demonstrate that the learned PDFA consistently avoids unsafe behaviors, while the synthesized strategies always satisfy the task requirements. Additionally, the robot's cost and task preferences remain within the bounds predicted by the corresponding Pareto point.


This manuscript substantially extends the plan synthesis component of the conference version \cite{watanabe2021probabilistic}.  Specifically, \cite{watanabe2021probabilistic} assumes a static environment and synthesizes a path over a graph that optimizes the preference measure of the learned PDFA.  This work generalizes \cite{watanabe2021probabilistic} by considering dynamic environments and synthesizing reactive strategies that guarantee the completion of the learned task while optimizing the trade-off between task preference measure and robot action cost. This significantly transforms the original problem from a graph search to a multi-objective game between the robot and the environment.  The new content includes a novel analysis method and synthesis algorithm for such games under deterministic strategies along with proofs of correctness and completeness as well as experimental evaluations. 

Overall, the contribution of this work is four-fold.
\begin{itemize}
    \item a derivation of a safety-guaranteed PDFA learning algorithm compatible with any EDSM techniques,
    \item a multi-objective reactive synthesis algorithm that leverages the learned PDFA to handle dynamic environments,
    \item a value iteration approach for Pareto front computation over deterministic strategies with completeness proof, and 
    \item a comprehensive set of experiments demonstrating the efficacy of the proposed algorithms in both mobile and manipulator robot applications.
\end{itemize}

\section{Related Work}
\textbf{Specification Learning: }
Many LfD research aims to learn a policy or a reward function.
For policy learning, techniques such as \textit{reinforcement learning} (RL) \cite{sutton2018reinforcement} and Dynamic Movement Primitives \cite{schaal2006dynamic, paraschos2013probabilistic} are typically used to learn a function that maps agent states to actions.
In reward learning, a scalar reward function that maps agent states to rewards is learned via, e.g., \textit{inverse reinforcement learning} (IRL)
\cite{ng2000algorithms, ziebart2008maximum, wulfmeier2015maximum, ramachandran2007bayesian}, to simultaneously train a policy on an agent. As mentioned above, these methods are fragile to the changes in the environment as they learn a function that is specific to the environment model used during training.

An alternative approach to expressing tasks is to use formal languages such as \textit{linear temporal logic} (LTL) \cite{BaierBook2008}, which is widely used in formal verification and increasingly employed in robotics in recent years, e.g., \cite{Lahijanian:AR-CRAS:2018}.  Such languages enable formal expression of rich missions, including non-Markovian tasks \cite{vazquez2017learning} as well as \textit{liveness} (``something good eventually happens'') and \textit{safety} (``something bad never happens'') requirements.
Other important benefits of formal languages is in their ease of interpretability and flexibility to compose multiple specifications.  Such benefits have even led to their use in RL, e.g., \cite{ijcai2019-840,li2017reinforcement,li2019formal}.
Nevertheless, writing correct formal specifications requires domain knowledge.

In recent years, a new line of research has emerged with a focus on learning formal specifications from data \cite{vazquez2017learning,vazquez2017logical,xu2018advisory,jha2017telex,shah2018bayesian}.
Most work has been concerned with learning temporal logic formulas with the purpose of classification and prediction from user data (in the supervised learning sense) \cite{xu2018advisory, jha2017telex} or interpretation and planning for tasks \cite{shah2018bayesian}.
Those studies restrict the exploration problem to a set of formula templates provided \textit{a priori}.
Recent work \cite{vazquez2017learning} overcomes this restriction by iterating over all combinations of formulas. The method is based on maximum a posterior learning and can account for noisy samples.  It however is slow due to the large space of exploration for formulas.
Another important issue with formula learning methods for the purpose of planning is that they typically need to be translated to an automaton, which could lead to the \emph{state-explosion} problem \cite{BaierBook2008, Lahijanian:AR-CRAS:2018}.
Work \cite{araki2019learning} overcomes this issue by directly learning a \textit{Deterministic Finite Automaton} (DFA).
They however assume the structure of the DFA is known and only learn the transitions between the DFA states while an oracle labels each sample with DFA states.

\textbf{Synthesis: }
Planning algorithms that utilize LTL have been widely explored \cite{Lahijanian:AR-CRAS:2018}, and this work builds upon these developed approaches. Common methods include automata-based techniques for discrete states \cite{Lahijanian:AR-CRAS:2018}, sampling-based motion planning \cite{bhatia2010sampling}, and reinforcement learning for continuous states \cite{camacho2019ltl}. These methods have been extended to synthesize solutions in reactive environments \cite{fainekos2005temporal}. In this study, we model a dynamic environment as a game between a system player and an environment player, similar to the frameworks presented in \cite{He:RAL:2019,he2017reactive,muvvala2022regret,muvvala2023efficient}. However, our approach must account for multiple quantitative objectives, rooted in the probabilities of the learned PDFA and the robot's operational costs, alongside the reachability requirement for task completion. This naturally leads to a multi-objective reachability game, where the goal is to synthesize a strategy that satisfies the reachability requirement while optimally trading off the quantitative objectives.


In \cite{chen2013synthesis, basset2015strategy, chen2013stochastic}, the authors explore multi-objective stochastic games, addressing both stopping and non-stopping games. For stopping games, they demonstrate that either infinite memory is necessary for deterministic strategies or randomization is required. In contrast, \cite{chatterjee2012strategy} presents a synthesis algorithm for multi-objective (multi-energy, mean-payoff, and parity) non-stopping games, showing that exponential memory is sufficient in multidimensional energy parity games and introducing a symbolic algorithm to compute a finite-memory winning strategy.
Our approach, however, synthesizes a deterministic strategy for a multi-objective stopping game using value iteration, which aligns more closely with \cite{chen2013synthesis}. Additionally, \cite{sastry2005new} addresses a similar problem but focuses on finding a set of Pareto-optimal solutions by transforming the multiple objectives into a single objective and reducing the problem to a shortest-path formulation. In contrast, our method computes the entire Pareto front, achievable by deterministic strategies.

\section{Preliminaries} \label{sec:preliminaries}


In this work, we are interested in deploying robots in dynamic environments to collect demonstrations for task learning and strategy synthesis. To define the problem, we first provide the necessary background on modeling the dynamic environments, demonstrations, task specifications, and strategies. Once all terms are defined, we formally introduce the problem in \Cref{sec:problemformulation}. 

\subsection{Two-player Game: Robot-Environment Interactive Model}

\rev{
We consider a robot that has to interact with a dynamic environment to achieve a task.  For example, the robot in \Cref{fig:introductory_example}, to fulfill its goal, has to interact with a school of fish that can freely move around. This interaction can be modeled as a game between the robot and the environment (fish), where each player has their own objectives and set of actions. 
While in reality, this game takes place in a continuous domain and may be concurrent, abstractions can be made to represent it as a discrete two-player game.  
Such an abstraction is commonly used and constructed in formal approaches to both mobile robotics \cite{Lahijanian:AR-CRAS:2018,Hadas:ICRA:2007,Lahijanian:ICRA:2009} and robotic manipulators \cite{He:ICRA:2015,He:RAL:2019,Muvvala:ICRA:2024}.
}

\begin{definition}[Two-player Game]
	A \textit{two-player game} is a tuple $\G = \DTGTuple$, where

	\begin{itemize}
		\item $\GState=\SysState \cup \EnvState$ is a finite set of states, where $\SysState$ and $\EnvState$ are the set of robot and environment states, respectively, and $\SysState \cap \EnvState = \emptyset$,
		\item $\GAction=\SysAction \cup \EnvAction$ is a finite set of controls or actions, where $\SysAction$ and $\EnvAction$ are the set of robot and environment actions, respectively,
		\item $\GInit \in \GState$ is the initial state,
		\item $\GDelta: \GState \times \GAction \rightarrow \GState$ is the transition function, 
		\item $\AP$ is a finite set of atomic propositions (predicates), 
		\item $\GLabel: \GState \rightarrow \PowerSetofAP$ is a labeling function that maps each state to the set of predicates that are true at that state, and
        \item $\GWeight: \GState \times \GAction \rightarrow \mathbb{R}^m_{\geq 0}$ is a weighting function that assigns to each  $(\Gstate,a) \in \GState \times \GAction$ an $m$-dimensional vector of non-negative weights $\GWeight(\Gstate,a)$.
	\end{itemize}
\end{definition}

Game $\G$ is also referred to as \textit{multi-objective} two-player game since it allows the encoding of multiple weights to each edge via $\GWeight$, i.e., $m$ weights and hence $m$ objectives. For instance, the weights could represent energy and distance costs.

The evolution of the game is as follows.  At state $\Gstate \in \GState_i$, player $i\in \{\Sys, \Env\}$ picks an action $a \in \GAction_i$, and receives a weight of $W(\Gstate,a)$. Then, the state of the game evolves to $s' = \GDelta(s,a) \in \GState_j, j \in \{\Sys, \Env\} \setminus \{i\}$. Next, it is player $j$'s turn to take an action, and the process repeats.  

\begin{myexam}[Two-player game]
The game abstraction of
\Cref{fig:introductory_example} is defined as follows. A state is a tuple of vehicle location $l_\Sys$, fish location $l_\Env$, and player's turn $i \in \{\Sys, \Env\}$, i.e., $s=(l_\Sys, l_\Env, i)$ where $l = (x, y)$ is the coordinate of each agent. Starting from the initial state $\GInit = ((2,1), (7,7), \Sys)$, the vehicle can take actions N (north), S (south), E (east), or W (west) to transition to an adjacent cell of distance 1 by consuming energy cost of 2, i.e., 
\rev{$\GWeight(s, a) = (1, 2)$ for all $a \in \{$N,\,S,\,E,\,W$\}$}
Fish can take action likewise.
The set of all possible atomic predicates that can be observed in this environment is $\AP = \set{\texttt{shipwreck}, \texttt{fish}, \texttt{coral-reefs}}$.
When the vehicle and fish are both at location $l_\Sys = l_\Env = (7, 7)$, then the observation is $\GLabel((l_\Sys, l_\Env, i)) = \{fish\}$ for $i\in\{\Sys,\Env\}$.
\end{myexam}


Players take actions in turn\footnote{
\rev{
Note that turn-taking occurs only in the (discrete) abstraction. In reality, agents may have continuous dynamics and act concurrently. The abstraction process maps these concurrent, continuous interactions into a discrete, turn-based game, as detailed in \cite{He:IROS:2017,Muvvala:ICRA:2024}.
%
}
} 
and this evolution results in a sequence of states called a play, which generates a sequence of observations called a trace (also known as word). 

\begin{definition}[Play \& Trace]
    \label{def:play}
    A \textit{play} $\play = \playSequence$ is a sequence of states starting from the initial state $s_0$, for all steps $0 \leq k < n$, there exists an action $a_k \in \GAction$ that transitions to the next state $s_{k+1} = \GDelta(s_k, a_{k})$. The set of all plays in $\G$ is denoted by $\PlaysOf$. 
    The output \textit{trace} of $\play$ is the sequence of state labels $\trace = \GLabel(\play) = \GLabel(s_0) \GLabel(s_1) \ldots \GLabel(s_n)$.
\end{definition}

    

The prefix of play $\play$ at position $k \leq n$ is a finite sequence of states $\play(k) = s_0 s_1 \ldots s_k$ from $s_0$ to the $k$-th state.
We say a prefix $\play(k)$ belongs to the robot player if $s_k \in \GState_\Sys$; otherwise, it belongs to the environment player.

\begin{myexam}[Play and Observation Trace]
Starting at $\GInit=((2,1), (7,7))$, if the robot takes actions S and W and fish takes action S in turn, then the resulting play is $\play = ((2,1),(7,7), \Sys)$ $((2,2), (7,7), \Env)$ $((2,2), (7,7), \Sys)$ $((1,2), (7,7), \Env)$. This induces the observation trace $\trace = \emptyset \emptyset \emptyset \set{\texttt{coral-reefs}}$.
\end{myexam}

The total cost that each player receives along a play is called
total payoff.

\begin{definition}[Multi-objective total payoff]
    \label{def:payoff}
    The \textit{multi-objective total payoff} of play $\play = s_0\ldots s_n$ is the sum of the weight vectors along $\play$, i.e,
    $$\totalpayoff(\play) = \sum_{k=0}^{n-1} \GWeight(s_k, s_{k+1}).$$
\end{definition}

In game $\G$, each player picks an action according to a strategy.  This choice of action can generally be deterministic or stochastic. 
In this work, we focus on deterministic strategies since we are interested in the robot behavior in one deployment instead of the expected behavior over multiple deployments.


\begin{definition}[Strategy]
    \label{def:strategy}
    \rev{Let $\GState^*\GState_i$ 
    denote the set of all finite plays that end in $\GState_i \subseteq S$, where $*$ is the Kleene star, and $i \in \{\Sys, \Env\}$.}
    A (deterministic) 
    \textit{strategy} for player $i \in \{\Sys, \Env\}$ is a function $\strategy_i: \GState^* \GState_i \rightarrow A_i$ that chooses the next action given a (finite) play that ends in a state in $S_i$.
\end{definition}

Under robot and environment strategies $\strategy_\Sys$ and $\strategy_\Env$, the game results in a single play denoted by $\PlaysOf(\strategy_\Sys, \strategy_\Env)$.  However, $\strategy_\Env$ is usually unknown; hence, our goal is to choose a $\strategy_\Sys$ that achieves the robot's objectives against all possible environment strategies, i.e., all possible plays under $\strategy_\Sys$, denoted by $\PlaysOf(\strategy_\Sys, \cdot)$. 

\subsection{Task Specifications}


We assume a robot task can be represented as a deterministic finite automaton (DFA) with alphabets $2^\Pi$.

\begin{figure}
    \begin{subfigure}{0.45\linewidth}
        \scalebox{.83}{
            \begin{tikzpicture}[
                ->, 
                >=stealth', 
                node distance=0.5\linewidth, 
                scale=0.8,
                every node/.style={scale=0.8, font=\small}]
            \node [state, initial] (q0) {$q0$};
            \node [state, below left of=q0] (q1) {$q1$};
            \node [state, below right of=q0] (q2) {$q2$};
            \node [state, accepting, below left of=q2] (q3) {$q3$};
            \draw
                (q0) edge[loop above, right, align=center] node[xshift=2mm, yshift=-2mm]{$\emptyset$} (q0)
                (q1) edge[loop left, below, align=center] node[xshift=-2mm, yshift=-1mm]{$\emptyset$} (q1)
                (q2) edge[loop right, below, align=center] node[xshift=2mm, yshift=-1mm]{$\emptyset$} (q2)
                (q0) edge[above left, align=center] node{$\set{\texttt{shipwreck}}$} (q1)
                (q0) edge[above right, align=center] node{$\set{\texttt{fish}}$} (q2)
                (q1) edge[below left, align=center] node[xshift=2mm, yshift=0mm]{$\set{\texttt{fish}}$} (q3)
                (q2) edge[below right, align=center] node[xshift=-2mm, yshift=0mm]{$\set{\texttt{shipwreck}}$} (q3);
            \end{tikzpicture}
        }
        \caption{Example DFA}
        \label{fig:exampleDFA}
    \end{subfigure}
    ~
    \begin{subfigure}{0.45\linewidth}
        \scalebox{.83}{
            \begin{tikzpicture}[
                ->, 
                >=stealth', 
                node distance=0.5\linewidth, 
                scale=0.8,
                every node/.style={scale=0.8, font=\small}]
            \node [state, initial] (q0) {$q0$};
            \node [state, below left of=q0] (q1) {$q1$};
            \node [state, below right of=q0] (q2) {$q2$};
            \node [state, accepting, below left of=q2] (q3) {$q3$:1.00};
            \draw
                (q0) edge[loop above, right, align=center] node[xshift=2mm, yshift=-2mm]{$\emptyset$: 0.8} (q0)
                (q1) edge[loop left, below, align=center] node[xshift=-2mm, yshift=-1mm]{$\emptyset$: 0.8} (q1)
                (q2) edge[loop right, below, align=center] node[xshift=2mm, yshift=-1mm]{$\emptyset$: 0.8} (q2)
                (q0) edge[above left, align=center] node{$\set{\texttt{shipwreck}}$: 0.15} (q1)
                (q0) edge[above right, align=center] node{$\set{\texttt{fish}}$: 0.05} (q2)
                (q1) edge[below left, align=center] node[xshift=2mm, yshift=0mm]{$\set{\texttt{fish}}$: 0.2} (q3)
                (q2) edge[below right, align=center] node[xshift=-2mm, yshift=0mm]{$\set{\texttt{shipwreck}}$: 0.2} (q3);
            \end{tikzpicture}
        }
        \caption{Example PDFA}
        \label{fig:examplePDFA}
    \end{subfigure}
    \caption{DFA and PDFA representation of an autonomous deep-sea science mission.}
    \label{fig:running_example_pdfa}
\end{figure}
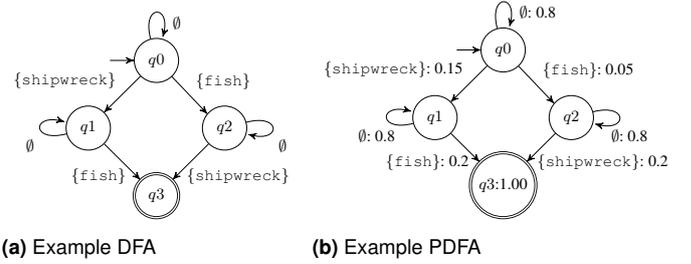

\begin{definition}[DFA]
	A \textit{deterministic finite automaton} (DFA) is a tuple $\A = (Q,\Sigma,q_0,\delta_\A,F)$, where
	\begin{itemize}
		\item $Q$ is a finite set of states,
		\item $\Sigma = 2^\AP$ is a finite set of input symbols, where each symbol is a subset of $\AP$,
		\item $q_0 \in Q$ is the initial state,
		\item $\delta_\A: Q \times \Sigma \to Q$ is the transition function, and
		\item $F \subseteq Q$ is the set of final or accepting states.
	\end{itemize}
\end{definition}

\noindent
The transition function $\delta_\A$ can be also viewed as a relation $\delta_\A \subseteq Q\times\Sigma\times Q$, where every transition is a tuple $(q,\sigma,q') \in \delta_\A$ iff $q' = \delta_\A(q,\sigma)$, where $\sigma \in \Sigma$.



A trace $\trace = \traceSequence$, where $\trace_i \in 2^\Pi$ for all $1 \leq i \leq n$, induces
a \textit{run} $\run = \runSequence$ on DFA $\A$, where $\run_0 = q_0$ and $\run_i = \delta(\run_{i-1},\trace_i)$ for $i = 1, \dots, n$.
A run $\run$ is called \textit{accepting} if $\run_n \in F$.  Trace $\trace$ is accepted by $\A$ if it induces an accepting run.  The set of all traces that are accepted by DFA $\A$ is called the language of $\A$ and is denoted by $\L(\A)$.
In game $\G$, we say play $\play$ satisfies the task represented by $\A$ if its output trace $\trace \in \L(\A)$.



\begin{myexam}
    \label{ex:DFA}
    \Cref{fig:exampleDFA}
    shows an example of a DFA that represents the robot task in \Cref{fig:introductory_example}.
    The set of accepting states is $F = \{q_3\}$.
    Trace $\omega = \emptyset \set{\texttt{shipwreck}} \emptyset \set{\texttt{fish}}$ induces accepting run $q_0 q_0 q_1 q_1 q_3$ on this DFA.
\end{myexam}



\rev{
A probabilistic extension of DFA is called PDFA, which assigns probabilities to the edges of the DFA~\cite{de2010grammatical}. 
}
This consequently induces a probability measure over the traces in the language of the DFA.  We use this measure as a preference metric over the accepting traces.

\begin{definition}[PDFA]
    \rev{
    A \textit{probabilistic DFA} (PDFA) is a tuple $\PA = (\A, \delta_\PP, F_\PP)$, where $\A$ is a DFA, and $\delta_\PP: 
    Q \times \Sigma \times Q    
    \to [0,1]$ assigns a probability to every transition in $\A$ such that
    $\sum_{\sigma \in \Sigma} \delta_\PP(q,\sigma,\delta_\A(q,\sigma))=1$
    for every $q\in Q$, and $F_\PP: Q \to [0,1]$ assigns a probability of terminating at each state such that $F_\PP(q) = 0$ if $q\not \in F$.
    }
\end{definition}

Consider trace $\trace = \traceSequence$ and its induced run $\run = \runSequence$ on PDFA $\PA$. The probability of $\trace$ is given by
$$P(\trace) = \prod_{i=1}^n \delta_\PP(\run_{i-1},\trace_i,\run_i) \cdot F_\PP(\run_n).$$
We say $\PA$ accepts $\trace$ iff $P(\trace) > 0$.
The language of $\PA$ is the set of traces with non-zero probabilities, i.e.,
$$\L(\PA) = \{\trace \in (\PowerSetofAP)^* \mid P(\trace) >0 \}.$$

\begin{myexam}[PDFA]
    \Cref{fig:examplePDFA}
    shows the PDFA extension of the DFA in \Cref{ex:DFA}, where the termination probability is 1 at $q_3$ and zero everywhere else.
    The probability of trace
    $$\omega = \emptyset \set{\texttt{shipwreck}} \emptyset \set{\texttt{fish}}$$
    is
    $P(\omega) = 0.8\times 0.15  \times 0.8 \times 0.2 \times 1.0 = 0.0192$.
\end{myexam}

\subsection{Safety Specifications}
\label{sec:safeLTL}
To express the safety constraints for the robot, we use safe LTL \cite{kupferman:FMSD:2001}, defined over the set of atomic propositions $\AP$.

\begin{definition}[Safe LTL Syntax]
	A \textit{syntactically safe LTL} formula over $\AP$ is recursively defined as
	\begin{equation*}
		\varphi := \ap \mid \lnot \ap \mid \varphi \lor \varphi \mid \varphi \land \varphi \mid  \mathcal{X} \varphi \mid \mathcal{G} \varphi
	\end{equation*}
	where $\ap \in \AP$, $\lnot$ (negation), $\lor$ (disjunction), and $\land$ (conjunction) are Boolean operators, and $\nextOp$ (``next'') and $\globallyOp$ (``globally'') are temporal operators.
\end{definition}

\rev{Note that the commonly used implication operator ($\rightarrow$) can be derived from $\neg$ and $\lor$, i.e., $\varphi_1 \rightarrow \varphi_2 \equiv \neg \varphi_1 \lor \varphi_2$.}

\begin{definition}[Safe LTL Semantics]
The semantics of syntactically safe LTL formulas are defined over infinite traces over $\PowerSetofAP$. Let $\infword = \word_1\word_2\ldots$ be an infinite trace $\infword \in (2^\Pi)^\omega$ with symbols $\word_i \in \PowerSetofAP$, and define
$\infword^i = \word_i \word_{i+1} \ldots$ to be a suffix of $\word$.
The notion $\infword \models \varphi$ indicates that $\infword$ satisfies formula $\varphi$ and is inductively defined as:
\begin{itemize}
    \item $\infword \models \ap$ iff $\ap \in \word_0$;
    \item $\infword \models \lnot \ap$ iff $\ap \notin \word_0$;

    \item $\infword \models \varphi_1 \lor \varphi_2$ iff $\infword \models \varphi_1$ or $\infword \models \varphi_2$;
    \item $\infword \models \varphi_1 \land \varphi_2$ iff $\infword \models \varphi_1$ and $\infword \models \varphi_2$;
    \item $\infword \models \nextOp \varphi$ iff $\infword^1 \models \varphi$;
    \item $\infword \models \globallyOp \varphi$ if $\forall k \geq 0$, $\infword^k \models \varphi$.
\end{itemize}
\end{definition}

Safe LTL formulas reason over infinite traces, but finite traces are sufficient to violate them \cite{kupferman:FMSD:2001}.
Therefore, as long as a finite trace does not violate the safe LTL, it respects the safety constraints.
We denote the set of finite traces that violate safety formula $\varphi_\safe$ by $\L(\neg \varphi_\safe)$. From safe LTL formula $\varphi_\safe$, we can construct a DFA $\A_{\neg \varphi_\safe}$ that accepts all violating traces $\L(\neg \varphi_\safe)$. By flipping the accepting condition and minimizing this DFA, we can construct a safety DFA $\A_{\varphi_\safe}$ where every state is accepting 
\rev{
(see~\cite{Lahijanian:TRO:2016} for more details).
}

\begin{figure}
    \centering
    \begin{subfigure}{0.6\linewidth}
        \centering
        \scalebox{.83}{
            \begin{tikzpicture}[
                ->, 
                >=stealth', 
                node distance=0.6\linewidth, 
                scale=0.8,
                every node/.style={scale=0.8, font=\small}]
              \node [state, initial] (q0) {$q0$};
              \node [state, right of=q0] (q1) {$q1$};
              \node [state, accepting, right of=q1] (q2) {$q2$};
              \draw
                (q0) edge[loop above, right, align=center] node[xshift=1mm, yshift=-1mm]{$\neg\set{\texttt{coral}}$} (q0)
                (q2) edge[loop above, right, align=center] node[xshift=1mm, yshift=-1mm]{true} (q2)
                (q0) edge[above, align=center] node{$\set{\texttt{coral}}$} (q1)
                (q1) edge[below, align=center, bend left] node{$\neg \set{\texttt{coral}}$} (q0)
                (q1) edge[above, align=center] node{$\set{\texttt{coral}}$} (q2);
            \end{tikzpicture} 
        }
        \caption{Safety Violating DFA $\A_{\neg \varphi_\safe}$}
        \label{fig:safetyViolatingDFA}
    \end{subfigure}
    
    \begin{subfigure}{0.6\linewidth}
        \centering
        \scalebox{.83}{
            \rev{
            \begin{tikzpicture}[
                ->, 
                >=stealth', 
                node distance=0.6\linewidth, 
                scale=0.8,
                every node/.style={scale=0.8, font=\small}]
              \node [state, accepting, initial] (q0) {$q0$};
              \node [state, accepting, right of=q0] (q1) {$q1$};
              \node [state, right of=q1] (q2) {$q2$};
              \draw
                (q0) edge[loop above, right, align=center] node[xshift=1mm, yshift=-1mm]{$\neg\set{\texttt{coral}}$} (q0)
                (q2) edge[loop above, right, align=center] node[xshift=1mm, yshift=-1mm]{true} (q2)
                (q0) edge[above, align=center] node{$\set{\texttt{coral}}$} (q1)
                (q1) edge[below, align=center, bend left] node{$\neg \set{\texttt{coral}}$} (q0)
                (q1) edge[above, align=center] node{$\set{\texttt{coral}}$} (q2);
            \end{tikzpicture} 
            }
        }
        \caption{\rev{Safety DFA $\neg \A_{\neg \varphi_\safe}$}}
        \label{fig:safetyNonViolatingDFA}
    \end{subfigure}

    \begin{subfigure}{0.6\linewidth}
        \centering
        \scalebox{.83}{
            \begin{tikzpicture}[
                ->, 
                >=stealth', 
                node distance=0.6\linewidth, 
                scale=0.8,
                every node/.style={scale=0.8, font=\small}]
              \node [state, initial, accepting] (q0) {$q0$};
              \node [state, accepting, right of=q0] (q1) {$q1$};
              \draw
                (q0) edge[loop above, right, align=center] node[xshift=1mm, yshift=-1mm]{$\neg\set{\texttt{coral}}$} (q0)
                (q1) edge[bend left, below, align=center] node[xshift=1mm, yshift=-1mm]{$\neg\set{\texttt{coral}}$} (q0)
                (q0) edge[above, align=center] node{$\set{\texttt{coral}}$} (q1);
            \end{tikzpicture}
        }
        \caption{Minimized Safety DFA $\A_{\varphi_\safe}$}
        \label{fig:safetyDFA}
    \end{subfigure}
    \caption{Safety Violating DFA, Safety DFA, and Minimized Safety DFA}
    \label{fig:safety_dfas}
\end{figure}
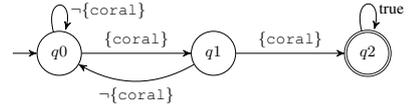
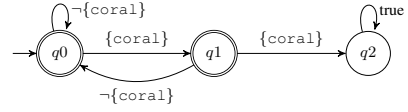
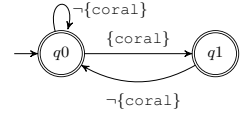

\begin{myexam}[Safe LTL]
    We can express the safety property ``always escape from coral reefs in 1 step''
    for the robot in \Cref{fig:introductory_example} as a safe LTL formula $\varphi = \globallyOp (\rm{coral} \rightarrow \nextOp \neg \rm{coral})$.  
    \rev{
    We construct the corresponding safety DFA $\A_{\varphi_\safe}$ by first constructing the safety-violating DFA $\A_{\neg {\varphi_\safe}}$ in \Cref{fig:safetyViolatingDFA}, and
    then by flipping the accepting condition (\Cref{fig:safetyNonViolatingDFA}) and minimizing it (\Cref{fig:safetyDFA}). 
    }
\end{myexam}

\section{Problem Formulation} \label{sec:problemformulation}

In this work, we are interested in deploying a robot in a dynamic environment with a task that is not specified but rather demonstrated, e.g., past deployment data.  Our goal for the robot is to infer the task and execute it according to the user preference with the completion and safety guarantees in face of environmental changes.  In addition to the challenges of dealing with task and preference inference and dynamic environment, we also do not require the demonstrations to be necessarily in the same environment in which the robot is to be deployed.  
This provides a high level of flexibility that allows the demonstration data to be collected independently from the environment and robot dynamics.
Below, we first introduce the notions of demonstrations, preferences, safety constraints, and task completion, and then introduce the formal statement of the problem.

We assume the demonstrator has a task $\varphi_\cosafe$ in mind that can be represented as a DFA with alphabet $2^\AP$.  Let $\A_{\cosafe}$ denote this DFA.  Based on $\varphi_\cosafe$, the demonstrator provides demonstrations to the robot.  We define a demonstration to be a sequence of symbols that achieves $\varphi_\cosafe$.
\begin{definition}[Demonstration]
    \label{def:demonstration}
    A \textit{demonstration} of task $\varphi_\cosafe$ is a trace $\trace = \trace_1 \trace_2 \ldots \trace_n$, where $\trace_i \in 2^\AP$ for all $1 \leq i \leq n$, such that $\trace \in \L(\A_{\cosafe})$.  
\end{definition}


Given a finite set of demonstrations $\traces =\set{\trace^1, \ldots, \trace^{n_{\traces}}}$, our goal is to learn the underlying task as well as the demonstrator's \emph{preferences} on how to achieve the task. For instance, a demonstrator may prefer to avoid a collision with an obstacle by \emph{steering left} since it may put the vehicle in a position that can avoid a future collision with less effort.  We assume that the preferred behaviors are demonstrated more often (repeated more) in $\traces$.
We aim to learn the task and preferences in the form of a PDFA $\tilde{\PA}_{\varphi_\cosafe}$, where the probabilities over traces reflect preferences. 

In addition, we specify a safety property $\varphi_\safe$ that the robot must not violate. The safety property characterizes a potentially infinite set of negative demonstrations that violate the safety property for our learner. By considering safety constraints, we also avoid trivial solutions to the problem, e.g., a trivial specification that accepts all possible behaviors. 

Now consider the interaction model of the robot and the environment as a two-player game $G$ and action costs for the robot. 
\rev{
Once a task is learned,
we are interested in synthesizing a deterministic strategy for the robot such that the robot is guaranteed to complete the learned task $\PA$ in one run while minimizing total cost (payoff) and maximizing the preferences.
}
The formal statement of the problem is as follows.

\begin{problem}
	\label{problem}
	Given robot-environment model as a two-player game $\G$ and a set of demonstrations $\traces =\set{\trace^i}_{i=1}^{n_\traces}$ 
	of a (latent) task $\varphi_\cosafe$ 
	according to some preference (probability distribution)
	and a safe LTL formula $\varphi_\safe$,
	\begin{enumerate}
		\item learn $\varphi_\cosafe$ and the demonstrator's preferences as a PDFA $\tilde \PA_{\varphi_\cosafe}$ such that it does not accept any trace that violates safety, 
		i.e, $\L (\tilde \PA_{\varphi_\cosafe}) \cap \L(\neg \varphi_\safe) = \emptyset$, and
		\label{prob:learning}

	

        \item compute a (deterministic) strategy $\strategy^{*}_\Sys$ such that every play under $\strategy^*_\Sys$ achieves $\varphi_\cosafe$ and never violates $\varphi_\safe$ while minimizing the worst total cost and maximizing the lowest preference, i.e., for $\play = \PlaysOf(\strategy_\Sys,\strategy_\Env)$,
        \begin{align}
            \begin{split}
                \label{eq:vector-optimization}
                \strategy^*_\Sys &= \argmin_{\strategy_\Sys} \max_{\strategy_\Env} \big(\totalpayoff (\play), \, - P ( \GLabel (\play)) \big)  
            \end{split}\\
            & \text{subject to } \nonumber \\
            \label{eq: satisfy PDFA}
            & \GLabel \big(\PlaysOf(\strategy_\Sys,\strategy_\Env)\big) \in \L (\tilde \PA_{\varphi_\cosafe}) \qquad \forall \strategy_\Env  \\
            \label{eq: satisfy safety}
            & \GLabel \big(\PlaysOf(\strategy_\Sys,\strategy_\Env)\big) \models \varphi_\safe \qquad \qquad \forall \strategy_\Env
        \end{align}
        
        where $P(\cdot)$ is the probability measure over PDFA $\tilde \PA_{\varphi_\cosafe}$.
        
		\label{prob:synthesis}
		
		
	\end{enumerate}
\end{problem}

Note that the optimization in \eqref{eq:vector-optimization} is multi-objective, i.e., $m+1$ objectives: $m$ dimensions of $\totalpayoff(\play)$ and one for $-P ( \GLabel (\play))$.  Further, these objectives are competing, i.e., 
by optimizing for one, the other becomes suboptimal.  In such cases, the interest is in the trade-offs of the objectives.  Hence, our goal is to \emph{optimize} for the trade-off between the objectives, which is known as \emph{Pareto optimal}.
Since there could be multiple optimal trade-offs, we aim to compute all Pareto optimal points possible under deterministic strategies.  Then, given a choice of one, we synthesize the corresponding strategy.  
\rev{
Finally, constraints \eqref{eq: satisfy PDFA}–\eqref{eq: satisfy safety} ensure that the computed strategy completes the learned task safely, regardless of the environment's strategy.
}



For Problem~1.1, we use grammatical inference~\cite{de2010grammatical} while  incorporating the safety property during the learning process, as described in \Cref{sec:pdfalearning}.  
We use this PDFA to solve Problem~1.2 as detailed in \Cref{sec:controlsynthesis}.

\section{Safety Guaranteed PDFA Learning}
\label{sec:pdfalearning}

In this section, we explain how a PDFA can be learned from demonstrations and present our method of embedding safety specification in the learning process. We first show a general PDFA learning algorithm, and then we describe our algorithms to incorporate safety.

\subsection{Grammatical Inference: PDFA Learning}

PDFA learning has been extensively studied as part of \emph{grammatical inference} (GI) with existing algorithms such as ALERGIA, DSAI, and  MDI, that can learn PDFAs from unlabeled demonstrations \cite{de2010grammatical}.
These algorithms are all based on a principle called \emph{evidence-driven state-merging} (EDSM).
At a high level, EDSM approaches find an  appropriate structure for an automaton $\PA$ and simultaneously estimate the probability distribution parameters $\FinalPA$ and $\deltaPA$ given a set of sample traces~$\traces$. 
This is  achieved by first constructing a large (prefix) tree from the samples, and then repeatedly  merging the states of the tree to form an 
automaton that is as compact as possible while ensuring the acceptance of the demonstrated traces in $\traces$. 
The difference between various algorithms (e.g., ALERGIA and  MDI) is in the choice of the method for merging states.

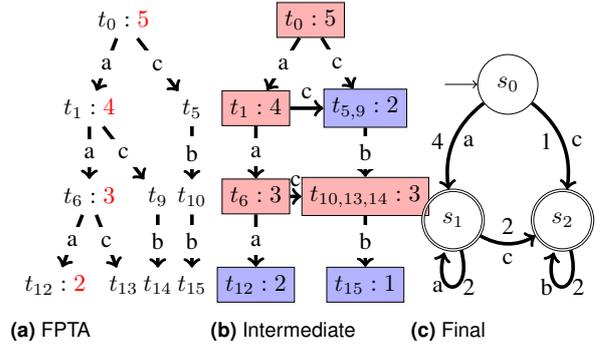
\begin{figure}
\centering
\begin{minipage}[b]{0.3\linewidth}
    \centering
    \scalebox{.9}{
        \begin{tikzpicture}
          \node (p0) at (-3,2) {$t_0: \textcolor{red}{5}$};
        
          \node (p1) at (-3.5,0.7){$t_1: \textcolor{red}{4}$};
          \node (p5) at (-2,0.7){$t_5$};
        
          \node (p6) at (-3.5,-0.6){$t_6: \textcolor{red}{3}$};
          \node (p9) at (-2.5,-0.6){$t_9$};
          \node (p10) at (-2,-0.6){$t_{10}$};
        
          \node (p12) at (-4,-1.9){$t_{12}: \textcolor{red}{2}$};
          \node (p13) at (-3,-1.9){$t_{13}$};
          \node (p14) at (-2.5,-1.9){$t_{14}$};
          \node (p15) at (-2,-1.9){$t_{15}$};
        
          \path[->, line width=1.5] (p0) edge node[midway, fill=white]{a} (p1)
          (p0) edge node[midway, fill=white]{c} (p5)
          (p1) edge node[midway, fill=white]{a} (p6)
          (p1) edge node[midway, fill=white]{c} (p9)
          (p5) edge node[midway, fill=white]{b} (p10)
          (p6) edge node[midway, fill=white]{a} (p12)
          (p6) edge node[midway, fill=white]{c} (p13)
          (p9) edge node[midway, fill=white]{b} (p14)
          (p10) edge node[midway, fill=white]{b} (p15);
        \end{tikzpicture}
    }
    \subcaption{FPTA}
    \label{fig:FPTA}
\end{minipage}
\begin{minipage}[b]{0.3\linewidth}
    \centering
    \scalebox{0.9}{
        \begin{tikzpicture}
          \node[rectangle, draw=black, fill=red!30] (q0) at (-0.2,2) {$t_{0}: 5$};
          \node[rectangle, draw=black, fill=red!30] (q1) at (-1,0.7){$t_{1}: 4$};
          \node[rectangle, draw=black, fill=blue!30] (q2) at (0.6,0.7){$t_{5,9}: 2$};
          \node[rectangle, draw=black, fill=red!30] (q3) at (-1,-0.6){$t_{6}: 3$};
          \node[rectangle, draw=black, fill=red!30] (q4) at (0.6,-0.6){$t_{10,13,14}: 3$};
          \node[rectangle, draw=black, fill=blue!30] (q5) at (-1,-1.9){$t_{12}: 2$};
          \node[rectangle, draw=black, fill=blue!30] (q6) at (0.6,-1.9){$t_{15}: 1$};
        
          \path[->, line width=1.5] (q0) edge  node[midway, fill=white]{a} (q1)
           (q0) edge  node[midway, fill=white]{c} (q2)
           (q1) edge  node[midway, fill=white]{a} (q3)
           (q2) edge  node[midway, fill=white]{b} (q4)
           (q1) edge node[above]{c} (q2)
           (q3) edge node[above]{c} (q4)
           (q3) edge node[midway, fill=white]{a} (q5)
           (q4) edge node[midway, fill=white]{b} (q6);
        \end{tikzpicture}
    }
    \subcaption{Intermediate}
    \label{fig:Itermediate}
\end{minipage}
\begin{minipage}[b]{0.3\linewidth}
    \centering
    \scalebox{0.9}{
        \begin{tikzpicture}

          \node[state,initial](n0) at (3,1.5) {$s_0$};
          \node[state,accepting](n1) at (2.2,-0.5) {$s_1$};
          \node[state,accepting](n2) at (3.8,-0.5) {$s_2$};
          \path[->,line width=1.5pt] (n0) edge[bend right] node[right]{a} node[left]{4} (n1)
          (n1) edge[bend right] node[below]{c} node[above]{2} (n2)
          (n0) edge[bend left] node[right]{c} node[left]{1}  (n2)
          (n1) edge [loop below] node[left]{a} node[right]{2} (n1)
          (n2) edge [loop below] node[left]{b} node[right]{2} (n2);
        
        \end{tikzpicture}
    }
    \subcaption{Final}
    \label{fig:Final}
\end{minipage}
\caption{Schematic illustration of evidence-driven state merging (EDSM) algorithm. (a) A frequency
    prefix tree acceptor (FPTA) is constructed from the given demonstrations with frequencies greater than 1 shown in red; (b) intermediate
    automaton as states are merged according to criteria that differ across various algorithms with frequencies shown at each node; and (c)
    the final frequency DFA (FDFA) that is learned is shown in red.}
\label{fig:EDSM-illustration}
\end{figure}

\Cref{fig:EDSM-illustration} shows a general scheme for an EDSM-based algorithm for learning a PDFA, and \Cref{alg:EDSM} shows the pseudo-code of this algorithm. The initial step is to construct a \textit{frequency prefix tree acceptor} (FPTA)
from the traces 
in $\traces$ (\Cref{fig:FPTA}, and \Cref{line:build-fpta} of \Cref{alg:EDSM}). The next step is to
incrementally merge states of the FPTA, two at a time, based on a
\emph{compatibility} criterion that varies depending on the actual
algorithm (Lines \ref{line:compatible} and \ref{line:merge}). As two states are merged, so are their subtrees in the FPTA 
(\Cref{fig:Itermediate}). The nodes of the
intermediate automata are variously
colored red/blue using a coloring scheme to keep track of how states are selected for merging.
Furthermore, algorithms also maintain frequencies alongside the nodes based on the number
of traces that reach a particular node. These frequencies are also combined
during the state merging process. The final result is a \textit{frequency DFA} (FDFA) wherein
frequencies along edges indicate how often they are taken by a demonstration (\Cref{fig:Final}). The frequencies
of all outgoing edges are normalized to yield a distribution (\Cref{line:convert}).

The various PDFA learning algorithms such as
ALERGIA or MDI differ on how they implement the \emph{compatibility}
check for whether two given nodes can be merged. For instance,
 the ALERGIA algorithm implements a statistical test based on frequencies to compare if two states are compatible, whereas the MDI approach
 first temporarily merges two states and their subtrees, and then checks if a metric computed on automaton  after the
 merge is smaller than that before the merge. If so, then it accepts the merge, otherwise, it rejects the merge.  
 Our goal is to learn a PDFA from $\traces$ while respecting safety property $\varphi_\safe$ by building on the existing PDFA learning algorithms, as described below.

\begin{algorithm}[t]
\SetKwInOut{Input}{Input}
\SetKwInOut{Output}{Output}
\Input{A demonstrated traces $\Omega$, and merge consistency parameter $\alpha$}
\Output{A PDFA}

$\mathcal{A}' \leftarrow \textsc{buildFPTA}(\Omega)$ \label{line:build-fpta}

$\RED \leftarrow \{q_0\}$

$\BLUE \leftarrow \textsc{children}(q_0)$

\While{ $|\BLUE|> 0$}{
    $q_b \leftarrow$ \textsc{choose}($\BLUE$)

    \eIf{$\exists q_r \in \RED \ \& \ \textsc{compatible}(\mathcal{A}', q_r, q_b, \alpha)$} { \label{line:compatible}
        $\mathcal{A}' \leftarrow $ $\textsc{stochasticMerge}(\mathcal{A}', q_r, q_b)$  \label{line:merge}
    }{
        $\RED \leftarrow \RED \cup \ \{ q_b \}$
    }
    $\BLUE \leftarrow \cup_{q \in \RED} \textsc{children}(q)  \setminus \RED$
}
\Return{$\textsc{FDFA2PDFA}(\mathcal{A}'$)} \label{line:convert}

\caption{\textsc{EDSM Algorithm}}
\label{alg:EDSM}
\end{algorithm}

\subsection{Learning with Safety Specification}

We now consider two different approaches for learning with a safety specification. 
The first method is a \emph{post-processing} technique that simply runs the PDFA learning algorithm on the given 
demonstration traces and then subsequently intersects the resulting PDFA with the automaton for the safety property.
The second method incorporates the safety specification during 
the learning process by modifying the EDSM algorithm. In particular, the merges are defined so that the result
continues to satisfy the safety specifications.

\subsubsection{Post-process Algorithm}
\label{sec:postprocess}

From $\varphiSafe$, we first construct a DFA $\A_{\neg \varphiSafe}$ that accepts precisely all those traces that violate the safety property~\cite{kupferman:FMSD:2001} (see \Cref{sec:safeLTL}). 
Then, by complementing $\A_{\neg \varphiSafe}$, we obtain  $\ASafe = \ASafeTuple$ that accepts all the traces that do not violate $\varphiSafe$.  
\rev{
Let  $\PA = \PATuple$ be the PDFA learned from the given demonstration traces without considering the safety property. 
}
We intersect the languages of $\PA$ and $\ASafe$ by constructing a product automaton $\PA_\safe$.

\begin{definition}[Product Automaton]
A product automaton is a tuple $\PA_\safe = \PA \otimes \ASafe = (Q_\safe, \Sigma, q_{0,\safe}, \delta_\safe, F_\safe, \delta_{\PP,\safe}, F_{\PP,\safe})$, where,

\begin{itemize}
	\item $Q_\safe = Q \times Q^s$,
	\item $q_{0,\safe} = (q_0, q_0^s)$, 
	\item $F_\safe = F \times F^s$,
    \item 
    \rev{
    $\delta_{\safe} ((q,q^s),\sigma) = (q',q^{s\prime})$ if $q' = \deltaA(q,\sigma) \wedge q^{s\prime} = \delta^s(q^s,\sigma)$,
    }
    \item $F_{\PP,\safe}((q,q^s)) = F_\PP(q)$,
    \item $\delta_{\PP,\safe}((q,q^s),\sigma,(q',q^{s\prime})) =$
    \begin{equation}
        \begin{cases}
            \frac{\delta_\PP(q,\sigma,q')}{N(q,q^s)} & \hspace{-1mm} \text{if } (q',q^{s\prime}) =  \delta_\safe((q,q^s),\sigma) \\
            0 & \hspace{-1mm} \text{otherwise}
        \end{cases}
        \label{eq:normalization}
    \end{equation}
\end{itemize}
where $N(q,q^s)$ is the normalizing function such that $\sum_{(\sigma, q_\safe') \in (\Sigma \times Q_\safe)} \delta_{\PP,\safe}((q,q^s),\sigma,q_\safe') = 1$.
\end{definition}

The resulting PDFA is guaranteed to be safe due to the intersection of languages. 
However, this method of pruning (imposing safety) as a post-process step 
alters the probability distributions  over the next-state transitions,  since we 
remove the transitions that violate safety and renormalize the probability distribution at each state, as shown in \eqref{eq:normalization}.
This overrides the probability distributions constructed by the original PDFA learning algorithm in an unpredictable manner.  Therefore,
while this method of imposing safety always succeeds, its probability distributions may not reflect the preferences embedded in the demonstrations accurately.

\subsubsection{Safety-Incorporated Learning Algorithm using ``Pre-Processing''} 
\label{sec:preprocess}

Whereas the \emph{post-processing} approach enforces safety after the PDFA is
learned, the pre-processing approach guarantees that the intermediate results always preserve safety, hence preventing alterations to the probability distributions due to unsafety.
The main idea is to build the PDFA that generalizes the demonstrated traces but carries along with it
the information about how the generalization satisfies the safety property $\varphiSafe$ at
the same time in the form of a simulation relation with $\ASafe$.

\begin{definition}[Simulation Relation]
A simulation relation $R$ 
between two automata $\scr{A}$ and $\scr{B}$
is a relation between their states, $R \subseteq Q_\scr{A} \times Q_\scr{B}$
\begin{compactenum}[(a)]
\item  Initial states of $\scr{A}$
relate to the initial states of $\scr{B}$; 
\item If pair $(s,t) \in R$, where $s \in Q_\scr{A}$ and $t \in Q_\scr{B}$, and automaton $\scr{A}$ can transition from $s$ to $s' \in Q_\scr{A}$ on symbol
$\sigma$, then there must exist a state $t' \in Q_\scr{B}$ such that automaton $\scr{B}$
transitions from $t$ to $t'$ on the same symbol $\sigma$ and $(s',t') \in R$; 
\item For each $(s, t) \in R$, if $s$ is final in 
$\scr{A}$ then $t$ must be final in $\scr{B}$.
\end{compactenum}
\end{definition}

\begin{theorem}
Let $R$ be a simulation relation between automata $\scr{A}$ and $\scr{B}$. It follows that $\L(\scr{A}) \subseteq \L(\scr{B})$.
\label{th:simulation-relation}
\end{theorem}
\begin{proof}
The proof is by induction on the string $\trace \in \Sigma^*$ that if $s \xrightarrow{\trace} s'$ and $(s, t) \in R$, then there exists $t'$ such that $t \xrightarrow{\trace} t'$ and $(s', t') \in R$. By definition of simulation relation, $(s_0, t_0) \in R$. For $\trace = \epsilon$ (empty string), $s = s_0$ and $t = t_0$. Since $(s, t) \in R$, it is trivial that $(s', t') \in R$. 
Now, assume the statement holds for a string $\trace$. Consider a string $\trace a$ where $a \in \Sigma$. Let the transitions be $s \xrightarrow{\trace} s' \xrightarrow{a} s''$ and $t \xrightarrow{\trace} t' \xrightarrow{a} t''$. By the induction hypothesis, $(s', t') \in R$. By the definition of simulation relation, $(s'', t'') \in R$. 

If $\trace \in \L(A)$, there exists $s' \in F_A$ such that $s \xrightarrow{\trace} s'$. By the simulation relation, there exists $t' \in F_B$ such that $t\xrightarrow{\trace} t'$. Thus, $\trace \in \L(B)$. Therefore, $\L(A) \subseteq \L(B)$.
\end{proof}

The proof simply shows by induction that for any accepting run corresponding to an input trace $\trace$ in automaton $\scr{A}$ from the initial state to a final state,  there
exists an accepting run  in $\scr{B}$ for the same trace $\trace$ from its initial state to the final state. The relation $R$ allows us to construct
such a run.

The key idea behind the \emph{pre-process} approach is to maintain a simulation relation between the FDFA and safety automaton $\ASafe$ at all intermediate states. The key is to restrict the merging of states so that  we can guarantee that a
simulation relation between the original automaton and $\ASafe$ before
merging can be modified to yield a simulation relation between the merged
automaton and $\ASafe$ afterwards.

Formally, we build the safety FPTA 
by augmenting the initial FPTA so
that each state is now a tuple of the form $(t_j, s_k)$ wherein $t_j$ is a node
in the original FPTA and $s_k$ is the state in $\ASafe$
reached when the prefix that leads upto the
state $t_j$ is run through $\ASafe$. 
Thus, we ensure that every branch not only corresponds to a demonstration but also to a valid trace in $\ASafe$.

Let $R$ be a relation between states of the FPTA and $\ASafe$ that contains all nodes $(t_j, s_k)$ in the safety FPTA.
\begin{lemma}
    \label{lemma1}
  Assuming no demonstration trace violates the safety property $\varphiSafe$, then $R$ is a simulation relation between the initial FPTA and the automaton $\ASafe$.
\end{lemma}

\begin{proof}
By definition, positive traces (in the initial FPTA) must be simulated by the automaton $\ASafe$. Thus, the FPTA and $\ASafe$ have a simulation relation.
\end{proof}

We can represent any intermediate FDFA state in the form $(T, s)$ wherein $T$
is a set of states from the initial FPTA, and $s$ is state in $\ASafe$. Next, we modify the EDSM approach to allow a merge between two states $(T_i, s_k)$ and $(T_j, s_l)$ only if $s_k = s_l$. The result of the merge creates a state
$(T_i \cup T_j, s_k)$.

\begin{lemma}
    \label{lemma2}
  Let $\scr{A}_1$ and $\scr{A}_2$ be the automata before and after an EDSM merge that
  is compatible with respect to the $\ASafe$ states. Let $R_1$ be the
  relation between the states of $\scr{A}_1$ and those of $\ASafe$ that
  is a simulation relation. 
  We can construct a simulation relation $R_2$ between the states of $\scr{A}_2$ and $\ASafe$.
\end{lemma}

\begin{proof}
Assume $\scr{A}_1$ and $\ASafe$ have a relation $R_1$. From \Cref{th:simulation-relation}, $\L(\A_1) \subseteq \L(\ASafe)$. Any transition in $\scr{A}_1$ $\bigl((s_{k-1} \rightarrow s_k)$, $(s_k \rightarrow s_k)$ and $(s_k \rightarrow s_{k+1})\bigr)$ can be simulated in $\ASafe$. Merging two nodes with the same safety states always keep these mappings. Therefore, any transition in $\scr{A}_2$ can be simulated in $\ASafe$. Thus, $R_2$ is a simulation relation.
\end{proof}

Combining Lemmas \ref{lemma1} and \ref{lemma2}, we conclude by induction on the number of merging steps that the final resulting PDFA must have a simulation relation to the safety automaton $\ASafe$.  Since we have a simulation relation, we conclude that the language of the final resulting FDFA and PDFA are contained in the that of $\ASafe$, i.e., the resulting PDFA does not accept a trace that violates $\varphiSafe$.

\section{Reactive Strategy Synthesis with PDFA}
\label{sec:controlsynthesis}

Once a task is learned as a PDFA $\tilde{\PA}$, we are interested in synthesizing a strategy to accomplish the learned task in a dynamic environment. 
At first, we reduce the reactive synthesis problem to a reachability game problem. 
To do so, we take a product of game $\G$ and PDFA $\tilde{\PA}$ to construct a Product Game that captures all possible plays that can achieve the task in $\G$. The strategy synthesis problem then turns into a quantitative (multi-objective) reachability game to guarantee that all the plays reach the accepting state with (Pareto optimal) payoffs. 


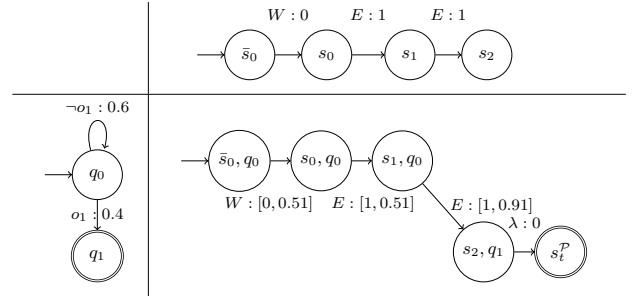
\begin{figure}
    \centering
    \scalebox{.75}{\begin{tabular}{c|c}
&
\begin{tikzpicture}[fontscale/.style={font=\small}]
    \matrix[every node/.style={font=\small}, row sep=15pt, column sep=13pt]{
        \node[state, initial](n12){$\augGInit$}; &
        \node[state, label={above right:{}}](n13){$s_0$}; &
        \node[state](n14){$s_1$}; &
        \node[state](n15){$s_2$}; &
        \node[draw=none](n16){}; &
        \node[draw=none](n17){}; \\
    };
    \draw[->,font=\footnotesize]
        (n12) edge[above] node[yshift=5mm]{$W:0$} (n13)
        (n13) edge[above] node[yshift=5mm]{$E:1$} (n14)
        (n14) edge[above] node[yshift=5mm]{$E:1$} (n15);
\end{tikzpicture} \\
\hline
\begin{tikzpicture}[fontscale/.style={font=\small}]
    \matrix[every node/.style={font=\small}, row sep=15pt, column sep=10pt]{
        \node(n21)[state, initial]{$q_0$}; \\
        \node(n31)[state, accepting]{$q_1$};\\
    };
    \draw[->,font=\footnotesize]
            (n21) edge[loop above] node{$\lnot o_1: 0.6$} (n21)
            (n21) edge node{$o_1: 0.4$} (n31);
\end{tikzpicture} &
\begin{tikzpicture}[fontscale/.style={font=\small}]
    \matrix[every node/.style={font=\small}, row sep=15pt, column sep=10pt]{
        \node[state, initial](n22){$\augGInit, q_0$}; &
        \node[state](n23){$s_0, q_0$}; &
        \node[state](n24){$s_1, q_0$}; &
        \node[draw=none](n25){}; &
        \node[draw=none](n26){}; & \\
        \node[draw=none](n32){}; &
        \node[draw=none](n33){}; &
        \node[draw=none](n34){}; & 
        \node[state](n35){$s_2, q_1$}; &
        \node[state, accepting](n36){$\GPdTerm$};\\
    };
    \draw[->,font=\footnotesize]
        (n22) edge[below] node[xshift=-2mm, yshift=-5mm]{$W:[0,0.51]$} (n23)
        (n23) edge[below] node[xshift=2mm, yshift=-5mm]{$E:[1,0.51]$} (n24)
        (n24) edge[right] node{$E:[1,0.91]$} (n35)
        (n35) edge[above] node[yshift=3mm]{$\lambda: 0$} (n36);
\end{tikzpicture}
\end{tabular}}
    \caption{
    Schematic of game product construction. The left automaton represents a PDFA $\tilde{\PA}$ and the top graph represents an augmented game graph $\augG$. The game product in the center is constructed by taking a product of the two. $E$ represents an action "East", $\lambda$ is a random action, and $o_1$ is an observation at $\GPdstate_2$ that satisfies the guard between $q_0$ and $q_1$ in $\tilde{\PA}$.}
    \label{fig:product_construction}
\end{figure}

\subsection{Product Construction}
First, we augment the game graph $\G$ with a new initial state $\augGInit$ and a transition to the original initial state $\GInit$.
Formally, we define $\augG = \augDTGTuple$, where $\augGState=\GState \cup \{\augGInit\}$, and $\augGDelta(s, a) = \GInit$ if $s=\augGInit$, otherwise $\augGDelta(s, a) = \GDelta(s, a)$ $\forall a \in \augGAction$.
This augmentation ensures that the label of $\GInit$ is correctly observed when taking a product with $\tilde{\PA}$.
An example of $\augG$ is shown in the top row of \Cref{fig:product_construction} and $\tilde{\PA}$ on the left column.
Given the two, we construct a \textit{multi-weighted game graph},

\begin{definition}[Product Game]
    Given augmented game $\augG = \augDTGTuple$ and PDFA $\tilde{\PA} = (Q,\Sigma,q_0,\delta_\A,F, \delta_\PP, F_\PP)$,
    the \textit{product game} is a tuple $ \GPd = \augG \times \tilde{\PA} = \GPdTuple$, where
	\begin{itemize}
		\item $\GPdState=(\augGState \times Q) \cup \{ \GPdTerm \} $ is a set of states (nodes),
		\item $\GPdAction$ is a finite set of controls or actions,
		\item $\GPdInit = (\augGInit, q_0) \in \GPdState$ is the initial state,
		\item $\GPdTerm \in \GPdState$ is the terminal state,
		\item $\GPdEdge \subseteq \GPdState \times \GPdAction \times \GPdState$ is a set of edges, and
		\item $\GPdWeight: \GPdEdge \rightarrow \mathbb{R}^{m+1}_{\geq 0}$ is a weighting function over edges.
	\end{itemize}
        The constructions of $\GPdEdge$ and $\GPdWeight$ are as follows.
    \begin{itemize}
        \item 
        \rev{
        Edge $e = ((s, q), a, (s', q')) \in \GPdEdge$ if $q'=\deltaA(q, \GLabel(s'))$ and $s'=\augGDelta(s, a)$.  Then, the multi-objective edge weight is $\GPdWeight((s, q), a, (s', q')) = \bigl(\GWeight(s, a), - \mbox{log}(\deltaPA(q, \GLabel(s'), q'))\bigr).$ 
        }
        \item Edge $e = ((s, q), a, \GPdTerm)) \in \GPdEdge$ \ if \ $\FinalPA(q) > 0$. Then, its weight $\GPdWeight((q, s), a, \GPdTerm) = \big(\vec{0}, -\mbox{log}(\FinalPA(q)) \big)$, where 
        $\vec{0}$ is a vector of zeros of length $m$.
    \end{itemize}
\end{definition}

Product game $\GPd$ captures the constraints of both the robot and task along with the robot/environment costs and demonstrator's preference. Let
$\pathGPd = \GPdInit \GPdstate_1 \ldots \GPdstate_n \GPdTerm = (\augGInit, q_0) (\GInit, q_1) \ldots (s_{n-1}, q_{n}) \GPdTerm$
be a path over $\GPd$.
The projection of this path (with the deletion of $\GPdTerm$) onto $\tilde{\PA}$ is an accepting run $q_0q_1\ldots q_n$. 
The projection of $\pathGPd$ on $\G$ is play $\play = s_0s_1\ldots s_{n-1}$ that generates the accepting observation trace  
$\trace = \observation = L(s_0)L(s_1) \ldots L(s_{n-1})$ that induces run $q_0q_1\ldots q_n$.
%
Furthermore, the total payoff of the path $\pathGPd$ is
\begin{align}
    \nonumber
    & \totalpayoff( \pathGPd ) \\
    \nonumber
    &= 
    \sum_{i=0}^{n-1} \GPdWeight(\GPdstate_{i}, a_i, \GPdstate_{i+1}) + \GPdWeight(\GPdstate_{n}, a_{n}, \GPdstate_{t}) \\
    \nonumber
    &= \Big(
    \sum_{i=0}^{n-1} \GWeight(s_i, a_i), \; -\sum_{i=0}^{n-1} \mbox{log} \big( \deltaPA(q_i, \GLabel(s_{i}), q_{i+1}) \big) -\\
    \nonumber
    & \hspace{57mm} \mbox{log}(\FinalPA(q_n)) \Big) \\
    \nonumber
    & = \Big(\totalpayoff (\play), \; -\mbox{log}\big(\prod_{i=0}^{n-1} \deltaPA(q_i, \GLabel(s_{i}), q_{i+1}) \cdot \FinalPA(q_n) \big) \Big) \\
    \label{eq: total cost of prod game}
    & = \Big(\totalpayoff (\play), \; - \mbox{log} \big( P ( L(\play) ) \big) \Big).
\end{align}


Therefore, to compute a robot strategy that produces accepting traces in $\L(\PA) \cap \L(\G)$, we need to find a strategy on $\GPd$, under which every play of the game for every environment strategy reaches the terminal state $\GPdTerm$.  Such a robot strategy is called \textit{winning}.
Specifically, among all the winning strategies, we require the ones that produce Pareto optimal costs; hence, solving Problem 1.2.

\subsection{Pareto Front Computation}
\label{sec:pareto points alg}

\rev{
Here, we focus on generating the set of all (Pareto) optimal values (Pareto front)
\cite{chen2013stochastic}.
}
First, we formally define Pareto front using the notion of dominance on vectors.
\begin{definition}[Dominance]
    \label{def:dominance}
    Given two vectors $v,v' \in \mathbb{R}^{m+1}_{\geq 0}$, we say 
    $v$ dominates $v'$, denoted by $v \succeq v'$, if $v_i \leq v'_i$ for every $0\leq i \leq m+1$, where $v_i$ is the $i$-th element of $v$.  Vector 
    $v$ \emph{strictly} dominates $v'$, denoted by $v \succ v',$ if $v_i < v_i'$ for all $0\leq i \leq m+1$.
\end{definition}
%
%
\begin{definition}[Pareto front]
    \label{def:pareto-front}
    Given a robot strategy $\strategy_\Sys$, denote the maximum total payoff that can be enforced by the environment by $v^*_{\strategy_\Sys}$, i.e.,
    $$ v^*_{\strategy_\Sys} = \max_{\strategy_\Env} \totalpayoff(\PlaysOf(\strategy_\Sys, \strategy_\Env)).$$
    We say that $\strategy_\Sys$ is a \emph{Pareto optimal} strategy if there does not exist another robot strategy $\strategy'_\Sys$ whose maximum total payoff $v^*_{\strategy'_\Sys}$ strictly dominates $v^*_{\strategy_\Sys}$, i.e.,
    $$ \nexists \strategy_\Sys' \; \text{ s.t. } \; v^*_{\strategy_\Sys'} \succ v^*_{\strategy_\Sys}.$$ 
    Then, $v^*_{\strategy_\Sys}$ is called a \emph{Pareto point}. The set of all Pareto points is called the \emph{Pareto front} $\ppSet$.
\end{definition}

With this definition and the total payoff equivalence in \eqref{eq: total cost of prod game}, Problem 1.2 reduces to generating the Pareto front and the corresponding optimal strategies on product game $\GPd$.  We first present an algorithm for Pareto front generation. Then, given a choice of a Pareto point, we can compute the corresponding strategy as discussed in Section~\ref{sec: strategy synthesis}.

Now, we present a polynomial algorithm to compute the Pareto points under all the winning strategies on $\GPd$ using a value iteration approach.  
The pseudocode is shown in \Cref{alg:pareto_points_computation}.
It uses the Pareto greatest fixed operator $\fpoperator_p$, defined below. 
Given set $V$, let $\ppSet(V) = \{v \in V \mid \nexists v' \in V \text{ s.t. } v' \succ v \}$ be the Pareto front of $V$, i.e., set of all Pareto points in $V$. Further, let $\tpSet(\GPdstate) \subset \mathbb{R}^{m+1}$ denote a set of total payoff vectors for plays initialized at state $\GPdstate$, and define the \textit{upper set} of vector $v \in \mathbb{R}^{m+1} \cup \{\vec{\infty}\}$ to be the set of vectors that are dominated by $v$, i.e., 
$$ \mathrm{upset}(v) = \{ v' \in \mathbb{R}^{m+1}_{\geq 0} \cup \{\vec{\infty}\} \mid v \succeq v' \}. $$ 
Then, we define $\fpoperator_\ppSet \big(\tpSet(\GPdstate) \big) = \ppSet (\fpoperator\big(\tpSet(\GPdstate) \big))$, where
\begin{align*}
    \label{eq: F-operator}
    &\fpoperator\big(\tpSet(\GPdstate) \big) = \nonumber \\
    &\left\{
    \begin{array}{ll}
     \!\!\!\! \bigcup\limits_{\substack{(\GPdstate, a, \GPdstateprime) \in \GPdEdge \\u \in \tpSet(\GPdstateprime)}} \!\!\!\!\!\!\!\! \mathrm{upset}(\GPdWeight(\GPdstate, a, \GPdstateprime) + u) & \text{if } \GPdstate \!\! \in \SysState \\
     \!\!\!\! \bigcap\limits_{\substack{(\GPdstate, a, \GPdstateprime) \in \GPdEdge, \\u \in \tpSet(\GPdstateprime)}} \!\!\!\!\!\!\!\! \mathrm{upset}(\GPdWeight(\GPdstate, a, \GPdstateprime) + u) & \text{if } \GPdstate \!\! \in \EnvState \\
    \end{array}
    \right.
\end{align*}

Intuitively, $\fpoperator_\ppSet$ back-propagates the set of total payoffs of the successor states and keeps only the Pareto optimal points. 
The upper set of the Pareto points represents the superset of all total payoffs that the robot can achieve. 
First, $\fpoperator$ takes the intersection/union operations on these sets to account for what is possible at the environment/robot nodes.  Specifically,
at the environment nodes, $\fpoperator$ takes the intersection of the upper sets to account for the worst (maximum) choice of the environment into account. This guarantees that the robot can always maintain the total costs lower than these values. Similarly, at the robot nodes, $\fpoperator$ takes the union of the sets to account for all of its choices.  Then, $\ppSet$ extracts the Pareto front of the resulting sets to complete one iteration of $\fpoperator_\ppSet$.  

A visualization of the above operation is shown in \Cref{fig:set_operation}.
The orange and green regions represent the upper set of the Pareto points at the successor nodes. The regions are then expanded by adding the edge weights. The left figure shows the union operation and the right figure shows the intersection operation.  The resulting sets are shown in gray.  Notice that the union operation maintains the largest possible set whereas the intersection operation shrinks the region. The resulting Pareto points are the vertices of the gray regions.

The algorithm initializes $\tpSet(\GPdstate)$ to a vector of zeros for the terminal state $\GPdstate = \GPdTerm$ and to infinity for all the other states $\GPdstate \in \GPdState \backslash \{\GPdTerm\}$. Then, it applies operator $\fpoperator_\ppSet$ recursively to back-propagate the Pareto points until the solution converges.
Note that applying $\fpoperator_p$ back-propagates costs from the terminating state to the initial state, exploring all plays and identifying all possible total payoffs.  The algorithm terminates when $\tpSet(\GPdstate)$ converges. 

Through this method, all those states that have finite values, i.e., $\max \tpSet(\GPdstate) < \infty $, are in the winning region, i.e., there exists a winning strategy under which all the plays initialized at $\GPdstate$ reach terminal state $\GPdTerm$.  For the rest of the states, a winning strategy does not exist.  Therefore, the obtained 
$\tpSet(\GPdInit)$, if finite, is the Pareto front for the initial states under only winning strategies, solving Problem~1.2.

\SetKwComment{Comment}{/* }{ */}
\begin{algorithm}[tb]
\SetKwInOut{Input}{Input}
\SetKwInOut{Output}{Output}
\Input{A Game Product Graph $\GPd = \GPdTuple$, Convergence margin $\epsilon$}
\Output{Pareto Points at $\GPdInit$}


$\tpSet(\GPdstate) \gets \{\vec{\infty}\} \;\; \ \forall \GPdstate \in \GPdState \backslash \{\GPdTerm\}$\; \label{line:pareto_initialization} 

$\tpSet(\GPdTerm) \gets \{\vec{0}\}$  \;




\While{$\tpSet$ is not converged 
}{



    \For{$\GPdstate \in \GPdState$}{

        $\tpSet'(\GPdstate) \gets \fpoperator_\ppSet (\tpSet(\GPdstate))$;
    
        
            
        
        
            
            
    }
    $\tpSet \gets \tpSet' $
}

\Return {$\tpSet(\GPdInit)$}
\caption{Pareto Points Computation}
\label{alg:pareto_points_computation}
\end{algorithm}

\begin{figure}
    \centering
    \begin{subfigure}[b]{0.45\linewidth}
        \centering
        \includegraphics[width=\linewidth]{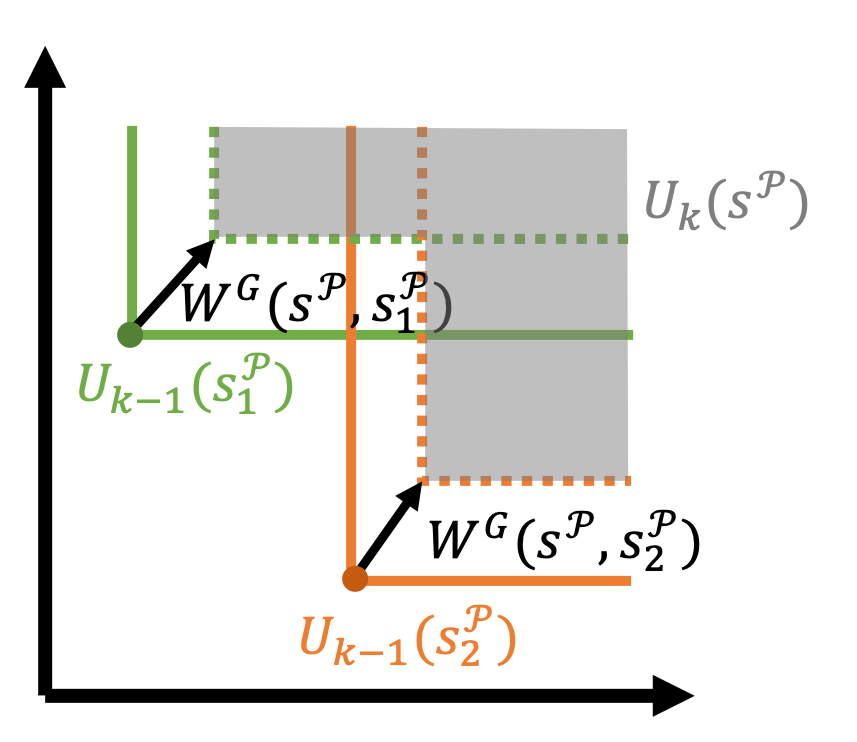}
        \caption{Union of two sets}
        \label{fig:union}
    \end{subfigure}
    \hfill
    \begin{subfigure}[b]{0.45\linewidth}
        \centering
        \includegraphics[width=\linewidth]{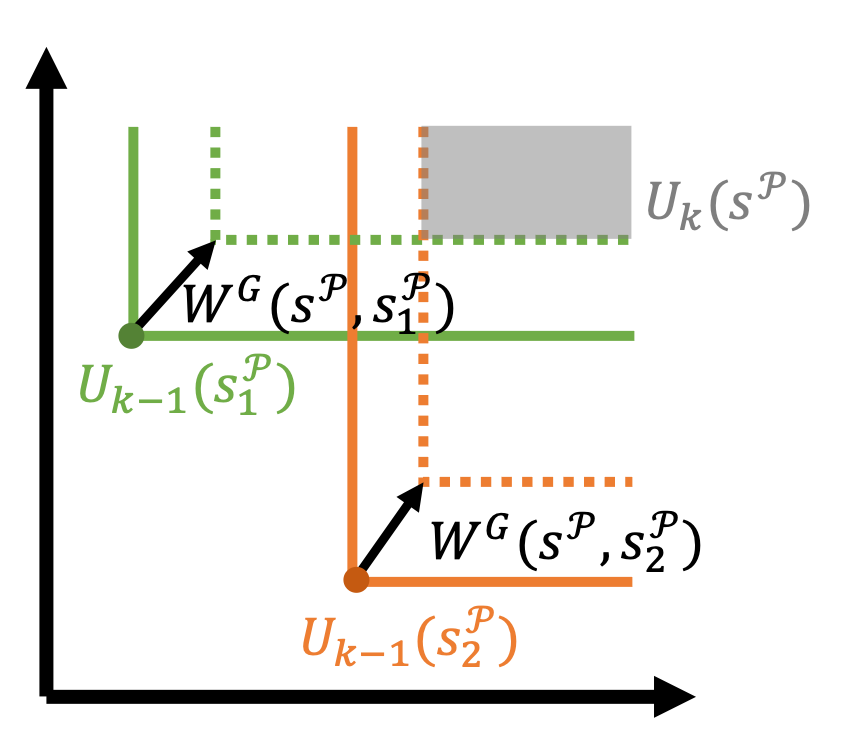}
        \caption{Intersection of two sets}
        \label{fig:intersection}
    \end{subfigure}
    \caption{The depiction of the two different set operations at node $s$ at iteration $k$ where $s_1$ and $s_2$ are its children.}
    \label{fig:set_operation}
    \vspace{-5mm}
\end{figure}

\begin{myexam}
\Cref{fig:pareto_points_compuation} illustrates an example of the product game. The values in square brackets represent edge weights, and those without any edge weights are assumed to be zero. 
Initially, the Pareto points at each node are set to infinities, while at terminating state $\GPdTerm$, they are set to zeros. 
The computation of Pareto points begins by applying $\fpoperator_p$ on $\GPdTerm$ and then recursively propagating back the Pareto front to the initial state $(\augGInit, q_0)$. 
In the first iteration, the total payoff at $(\Gstate_9, q_1)$ gets updated from infinities to zeros (zero weights and zero costs at $\GPdTerm$). 
\rev{
Next, those at nodes $(\Gstate_i, q_0)$ for $i\in \{5,\ldots,8\}$ are also updated to zeros.  
Note that, at $(\Gstate_4, q_0)$, the environment player can force a self-loop, ensuring a win against the system. This results in a value of $\tpSet((s_4,q_0)) = \{\vec{\infty}\}$.}

At $(\Gstate_2, q_0)$, the system can choose $(\Gstate_4, q_0)$, $(\Gstate_5, q_0)$, or $(\Gstate_6, q_0)$.  
By taking the union of the upper sets of the Pareto points of the successor states, we obtain the Pareto point of $(5,5)$. Note that by taking the union, infinity values are subsumed in the upper set of the point $(5, 5)$. In other words, $(5, 5)$ dominates the infinity values so the infinity costs are no longer considered. Similarly, the Pareto points at $(\Gstate_3, q_0)$ are (1,10) and (10,1). 
\rev{
In the next iteration, at environment state $(\Gstate_1, q_0)$, 
we take the intersection of the upper sets of $(5, 5)$ with the union of the upper sets of $(1,10)$ and $(10,1)$, resulting in Pareto points $(5, 10)$ and $(10,5)$.  In the two iterations, those values are propagated to $(\augGInit,q)$, and convergence is achieved.
}
\end{myexam}

\begin{figure}
    \centering
    \scalebox{.75}{\begin{tikzpicture}[
fontscale/.style = {font=\small},
node distance=40pt and 50pt,
every node/.style={}
]
\node[state, circle, initial, label={below:$\{(5,10),(10,5)\}$}](n1){$\augGInit, q_0$};
\node[state, circle, right of=n1](n2){$s_0, q_0$};
\node[state, rectangle, right of=n2, label={above:$\{(5,10),(10,5)\}$}](n3){$s_1, q_0$};
\node[state, circle, above right of=n3, yshift=5mm, label={above:$\{(5,5)\}$}](n4){$s_2, q_0$};
\node[state, circle, below right of=n3, yshift=-5mm, label={below:$\{(1,10),(10,1)\}$}](n5){$s_3, q_0$};
\node[state, rectangle, xshift=5mm, right of=n4](n7){$s_5, q_0$};
\node[state, rectangle, above of=n7, yshift=-3mm](n6){$s_4, q_0$};
\node[state, rectangle, below of=n7, yshift=3mm](n8){$s_6, q_0$};
\node[draw=none, xshift=5mm, right of=n5](n0){};
\node[state, rectangle, above of=n0, yshift=-5mm](n9){$s_7, q_0$};
\node[state, rectangle, below of=n0, yshift=5mm](n10){$s_8, q_0$};
\node[state, circle, right of=n8](n11){$s_9, q_1$};
\node[state, circle, accepting, right of=n11](n12){$\GPdTerm$};

\draw[font=\footnotesize]
    (n1) edge (n2)
    (n2) edge (n3)
    (n3) edge (n4)
    (n3) edge (n5)
    (n4) edge (n6)
    (n4) edge[above] node{$[5,5]$} (n7)
    (n4) edge[below] node{$[5,5]$} (n8)
    (n5) edge[above] node{$[10,1]$} (n9)
    (n5) edge[above] node{$[1,10]$} (n10)
    (n6) edge[loop above] (n6)
    (n6) edge (n11)
    (n7) edge (n11)
    (n8) edge (n11)
    (n9) edge (n11)
    (n10) edge (n11)
    (n11) edge (n12)
    ;
\end{tikzpicture}}
    \caption{
    \rev{
    Schematic of Pareto point computation.
    Circle and square nodes represent system and environment states, respectively.
    Edge weights are shown in square brackets, and Pareto points at nodes are enclosed in curly brackets.
    Edges without displayed weights are assumed to have weight zero.}
    }
    \label{fig:pareto_points_compuation}
\end{figure}
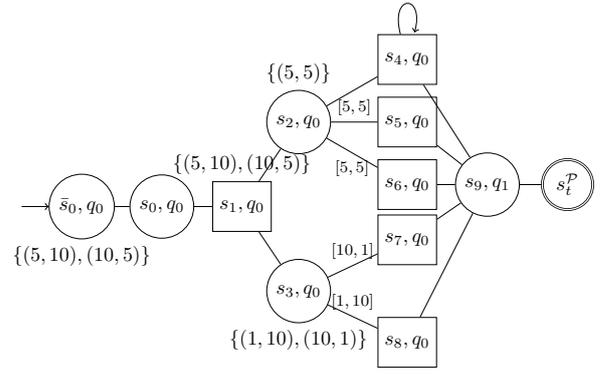

Below, we prove the completeness and runtime complexity of this algorithm. 
Specifically, we show that $\fpoperator_\ppSet$ operator is contractive and its greatest fixed-point is the true Pareto front and achieved in finite time. For simplicity, we overload the dominance operator $\succeq$ to apply to sets of total payoffs, i.e., we write $U' \succeq U$ if $u' \succeq u$ for all $u' \in U'$ and for all $u \in U$. 

\begin{lemma}
Given game product $\GPd$ and operator $\fpoperator_\ppSet$, let $\tpSetAfter{i}(\GPdstate)$ denote the set of total payoff points at state $\GPdstate \in \GPdState$ obtained in the $i$-th iteration of Alg.~\ref{alg:pareto_points_computation}.  Then, the total payoff points in $\tpSetAfter{i+1}(\GPdstate) = \fpoperator_\ppSet(\tpSetAfter{i}(\GPdstate))$ dominate the total payoff points in $\tpSetAfter{i}(\GPdstate)$ for every $\GPdstate \in \GPdState$, 
i.e., 
$$ \tpSetAfter{i+1}(\GPdstate) \succeq \tpSetAfter{i}(\GPdstate) \qquad \forall \GPdstate \in \GPdState.$$
%
%
%
\label{lemma: monotonic}
\end{lemma}

The proof is provided in Appendix~\ref{appendix: proof of lemma monotonic}. Lemma~\ref{lemma: monotonic} shows that operator $\fpoperator_\ppSet$ is monotonic in that the newly computed total payoff set dominates the previous one.  Hence, by repeatedly applying $\fpoperator_\ppSet$, we prune out dominated values.  The following proposition shows that in finite number of iterations of applying $\fpoperator_\ppSet$, the total payoff sets can be propagated to all the states.


\begin{proposition}
    The set of total payoffs can be propagated to every state in $\GPdComplexity$ iterations. 
    \label{lemma:maxpropagation}
\end{proposition}
%

The proof is provided in Appendix~\ref{appendix: proof of lemma:maxpropagation}.  Using the results of this proposition, the following lemma shows that $\fpoperator_\ppSet$ reaches a fixed-point in polynomial time.

\begin{lemma}
    After the maximum iterations in \Cref{lemma:maxpropagation}, $\tpSetAfter{i}$ does not change. 
    \label{lemma:convergence}
\end{lemma}
\begin{proof}
    This can be shown with a contradiction. 
    Recall that the set of payoffs obtained by $\fpoperator_\ppSet(\tpSetAfter{i}(\GPdstate))$ changes only if a shorter path to a terminal state is found.  However, by \Cref{lemma:maxpropagation} all the paths with dominant payoffs are explored in at most $\GPdComplexity$ iterations.  
    After $\GPdComplexity$ iterations, if there were a shorter path, then it would have to contain an unvisited successor node ${\GPdstate}'$. 
    Let the current state be $\GPdstate$ and the total payoffs from $\GPdTerm$ to states $\GPdstate$ and ${\GPdstate}'$ at step $i$ be $\tpSetAfter{i}({\GPdstate}')$ and $\tpSetAfter{i}({\GPdstate}')$, respectively. 
    \rev{
    If ${\GPdstate}'$ were in the shortest path, then $\tpSetAfter{i}(\GPdstate) \preceq  \tpSetAfter{i}({\GPdstate}') \oplus \{W({\GPdstate}', a, \GPdstate)\}$. 
    However, $\tpSetAfter{i}(\GPdstate)$ must be dominant over $\tpSetAfter{i}({\GPdstate}') \oplus \{W({\GPdstate}', a, \GPdstate)\}$
    }
    because otherwise the total payoffs would have been included in the upper set, i.e., $\tpSetAfter{i}(\GPdstate) \supseteq \mathrm{upset}(W({\GPdstate}', a, \GPdstate) + u)$ for $u \in \tpSetAfter{i}({\GPdstate}')$. Hence,  $\tpSetAfter{i}$ converges in $\GPdComplexity$ iterations. 
\end{proof}

The above lemmas and proposition show that the Pareto front is the greatest fixed-point of $\fpoperator_\ppSet$, which can be computed in finite time.  Hence, \Cref{alg:pareto_points_computation} is complete and efficient as stated below.

\begin{theorem}
Given a game $\GPd$, \Cref{alg:pareto_points_computation} computes the Pareto fronts for all states in $\GPd$ in polynomial time.
\label{thm: pareto_points_computation}
\end{theorem}
\begin{proof}
    By initializing the Pareto sets to vectors of infinity,
    \Cref{alg:pareto_points_computation} correctly computes the Pareto points at each iteration $i \geq 0$ by \Cref{lemma: monotonic}. By \Cref{lemma:maxpropagation}, the algorithm propagates all edge weights after $\GPdComplexity$ iterations, 
    leading to the convergence of the total payoff sets by \Cref{lemma:convergence}. 
\end{proof}


\begin{remark}
    While Theorem~\ref{thm: pareto_points_computation} shows that \Cref{alg:pareto_points_computation} is guaranteed to terminate in finite number of iterations,  in each iteration, the set operations of intersection and union need to be performed.  In our implementation, 
    we use
    the polygon clipping algorithm 
    which runs in $\mathcal{O}(n\cdot m)$, where $n$ and $m$ are the number of vertices of the two polygons \cite{puri2013efficient}. 
    This algorithm is known to run in $\mathcal{O}((k + n) \log n)$ where $k$ is a number of intersections. 
\end{remark}

\subsection{Pareto Optimal Strategy Synthesis}
\label{sec: strategy synthesis}

\begin{algorithm}[tb]
\SetKwInOut{Input}{Input}
\SetKwInOut{Output}{Output}
\Input{A game product graph $\GPd = \GPdTuple$, a Pareto point at the initial state $\pp$, and the sets of total payoffs $\tpSet$}
\Output{A deterministic strategy $\strategy$}

$Q = Queue((\GPdInit, \pp))$\label{str: init}\; 

$Visited = Set(\GPdInit)$\;

$E = List()$\;

\While{$Q$ is not empty}{

    $\GPdstate, \ \pp \leftarrow Q.pop()$\;
 
    \For{$(\GPdstate, a, \GPdstateprime) \in \GPdEdge$}{
        
        \If{$\pp' \in \tpSet(\GPdstateprime)$ and $\pp - \GPdWeight(\GPdstate, a, \GPdstateprime) \geq \pp'$}{  \label{str: inppset}
        
            $\strategy(\GPdstate) = a$\;
                    
            \If{$\GPdstateprime$ not in $Visited$}{
            
                $Q.add((\GPdstateprime, \pp'))$\;
            
                $Visited.add(\GPdstateprime)$\;
                
                $E.append((\GPdstate, a, \GPdstateprime))$\;
                
            }
        }
    }

}

\Return{$\strategy$}
\caption{Strategy Synthesis for a Pareto point $p$}
\label{alg:strategy_synthesis}
\end{algorithm}


Here, we show that a strategy for a selected Pareto point can be computed in linear time with respect to the number of nodes in the product game. The algorithm is presented in \Cref{alg:strategy_synthesis}.
At a high level, the algorithm selects an action at each state to find paths whose total payoff is less than or equal to the Pareto point. Recall that Pareto points represent the worst possible total payoffs. Starting from the initial state, any path leading to a terminating state must have a cost less than or equal to the Pareto point. Thus, as long as the difference between the Pareto point and the accumulated path cost at the current node is positive, the selected action ensures that the total payoff remains within the Pareto point. The algorithm tracks the \textit{remaining} total payoff starting from the initial state.

On \Cref{str: init}, we start by adding the initial state and the selected Pareto point $\pp \in \ppSet$ to a queue. On \Cref{str: inppset}, we choose a successor node $\GPdstate$ such that the successor's Pareto point $\pp'$ remains inside the current node's cost set, $\pp-\GPdWeight(\GPdstate, a, \GPdstate_i)$. This ensures that the remaining total payoff is positive, $\pp-\GPdWeight(\GPdstate, a, \GPdstate_i) - \pp' \geq 0$. 
Once the strategy is obtained, the algorithm checks if there are possible loops. For this, we can construct a strategy graph by only keeping the strategy's actions in the product game, running a backward reachability analysis on the strategy graph, and only retaining the states that are reachable from the accepting state. This prevents cycling in a loop in the strategy graph.

\begin{figure*}[t]
    \centering
    \begin{subfigure}[b]{0.18\textwidth}
        \centering
        \includegraphics[width=.82\textwidth]{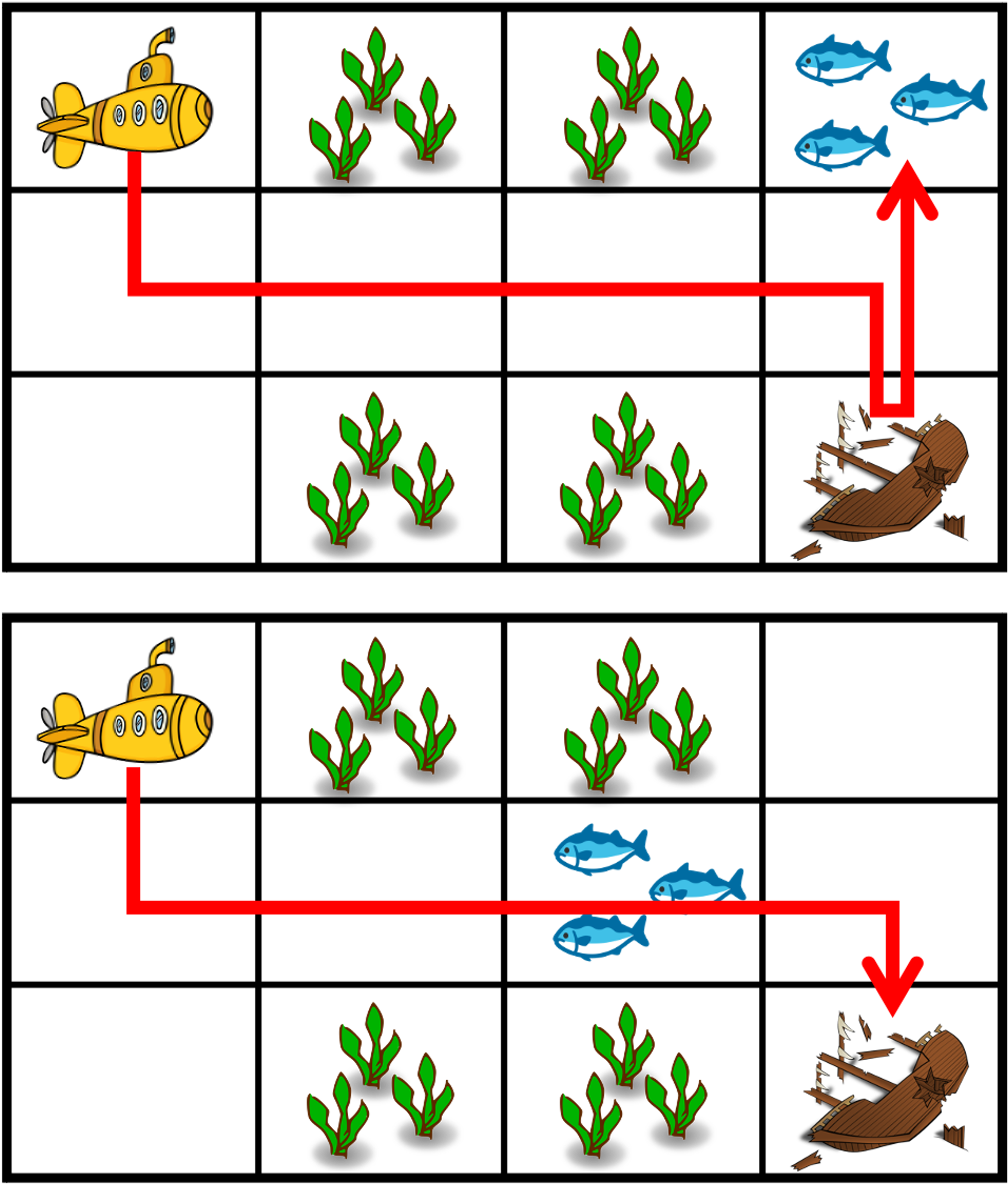}
        \caption{Non-Markovian Task}
        \label{fig:reachTwoGoals}
    \end{subfigure}
    \hfill
    \begin{subfigure}[b]{0.18\textwidth}
        \includegraphics[width=\textwidth]{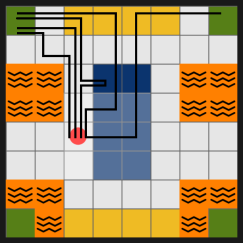}
        \caption{Five demos}
        \label{fig:seshia}
    \end{subfigure}
    \hfill
    \begin{subfigure}[b]{0.18\textwidth}
        \includegraphics[width=\textwidth]{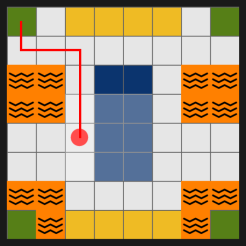}
        \caption{Synthesized Plan}
        \label{fig:gridworld1}
    \end{subfigure}
    \hfill
    \begin{subfigure}[b]{0.18\textwidth}
        \includegraphics[width=\textwidth]{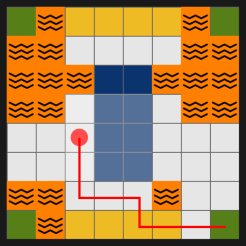}
        \caption{Synthesized Plan}
        \label{fig:gridworld2}
    \end{subfigure}
    \hfill
    \begin{subfigure}[b]{0.23\textwidth}
        \includegraphics[width=\textwidth]{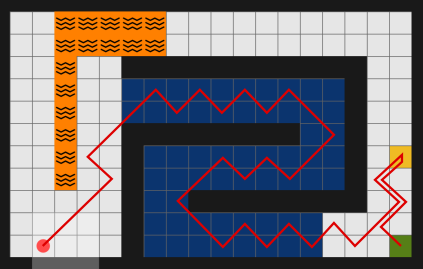}
        \caption{Synthesized Plan for a diagonal-moving robot}
        \label{fig:gridworld3}
    \end{subfigure}
    \caption{Various environments and robots considered for the case studies. (a) Learning and planning for the non-Markovian task. (b) Environment and demonstrations from \cite{vazquez2017learning}. (c)-(e) Synthesized plans (shown in red) based on the learned task from (b).}
    \label{fig:gridworld}
    \vspace{-3mm}
\end{figure*}
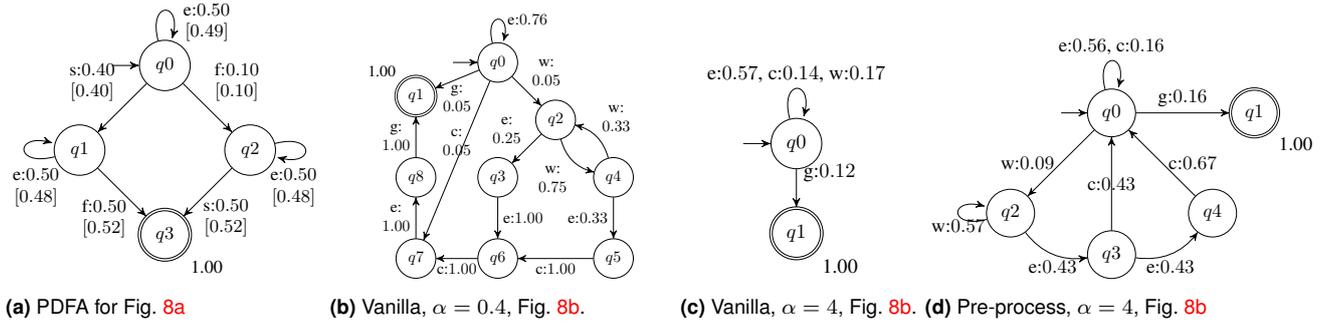
\begin{figure*}
\centering
\begin{minipage}[b]{0.24\linewidth}
    \centering
    \scalebox{.95}{

\begin{tikzpicture}[
    ->, 
    >=stealth', 
    node distance=0.5\linewidth, 
    scale=0.8,
    every node/.style={scale=0.8, font=\small},
    ]

  \node [state, initial] (q0) {$q0$};
  
  \node [state, below left of=q0] (q1) {$q1$};
  
  \node [state, below right of=q0] (q2) {$q2$};
  
  \node [state, accepting, below left of=q2, label=below right:1.00] (q3) {$q3$};

  \draw
    (q0) edge[loop above, right, align=center] node[xshift=2mm, yshift=-2mm]{\emptyAP:$0.50$\\\target{0.49}} (q0)
    
    (q1) edge[loop left, below, align=center] node[xshift=2mm, yshift=-2mm]{\emptyAP:$0.50$\\\target{0.48}} (q1)

    (q2) edge[loop right, below, align=center] node[xshift=-2mm, yshift=-2mm]{\emptyAP:$0.50$\\\target{0.48}} (q2)
    
    (q0) edge[above left, align=center] node{ \ship:$0.40$\\\target{0.40}} (q1)
    (q0) edge[above right, align=center] node{\fish:$0.10$\\\target{0.10}} (q2)
    (q1) edge[below left, align=center] node[xshift=2mm, yshift=0mm]{\fish:$0.50$\\\target{0.52}} (q3)
    (q2) edge[below right, align=center] node[xshift=-2mm, yshift=0mm]{\ship:$0.50$\\\target{0.52}} (q3)

    ;
\end{tikzpicture}}
    \subcaption{PDFA for Fig. \ref{fig:reachTwoGoals}}
    \label{fig:PDFA_twogoals}
\end{minipage}
\begin{minipage}[b]{0.26\linewidth}
    \centering
    \scalebox{0.85}{

\begin{tikzpicture}[
    ->, 
    >=stealth', 
    node distance=0.4\linewidth, 
    scale=0.8,
    every node/.style={scale=0.7},
    ]

  \node [state, initial] (q0) {$q0$};
  
  \node [state, below right of=q0] (q2) {$q2$};
  
  \node [state, below left of=q2] (q3) {$q3$};
  
  \node [state, below right of=q2] (q4) {$q4$};
  
  \node [state, below of=q4] (q5) {$q5$};
  
  \node [state, below of=q3] (q6) {$q6$};
  
  \node [state, left of=q6] (q7) {$q7$};
  
  \node [state, above of=q7] (q8) {$q8$};
  
  \node [state, accepting, above of=q8, label=above left:1.00] (q1) {$q1$};

  \draw
    (q0) edge[loop above, right] node{\hspace{1mm}\emptyAP:$0.76$} (q0)
    (q0) edge[below, align=center] node{\green:\\$0.05$} (q1)
    (q0) edge[above, align=center] node{\carpet:\\$0.05$} (q7)
    (q0) edge[above right, align=center] node{\water:\\$0.05$} (q2)

    (q2) edge[above left, align=center] node{\emptyAP:\\$0.25$} (q3)
    (q2) edge[bend right, below left, align=center] node{\water:\\$0.75$} (q4)

    (q3) edge[right] node{\emptyAP:$1.00$} (q6)

    (q4) edge[bend right, above right, align=center] node{\water:\\$0.33$} (q2)
    (q4) edge[left] node{\emptyAP:$0.33$} (q5)

    (q5) edge[below] node{\carpet:$1.00$} (q6)

    (q6) edge[below, align=center] node{\carpet:$1.00$} (q7)

    (q7) edge[left, align=center] node{\emptyAP:\\$1.00$} (q8)

    (q8) edge[left, align=center] node{\green:\\$1.00$} (q1)
    ;
\end{tikzpicture}}
    \subcaption{Vanilla, $\alpha=0.4$, Fig. \ref{fig:seshia}.}
    \label{fig:experiment1result}
\end{minipage}
\begin{minipage}[b]{0.18\linewidth}
    \centering
    \scalebox{0.95}{

\begin{tikzpicture}[
    ->, 
    >=stealth', 
    node distance=0.5\linewidth, 
    scale=0.8,
    every node/.style={scale=0.8}
    ]

  \node [state, initial] (q0) {$q0$};
  
  \node [state, accepting, below of=q0, label=below right:1.00] (q1) {$q1$};

  \draw
    (q0) edge[loop above, above] node{\emptyAP:$0.57$, \carpet:$0.14$, \water:$0.17$} (q0)
    (q0) edge[above right, align=center] node{\green:$0.12$} (q1)
    ;
\end{tikzpicture}}
    \subcaption{Vanilla, $\alpha=4$, Fig. \ref{fig:seshia}.}
    \label{fig:spec}
\end{minipage}
\begin{minipage}[b]{0.3\linewidth}
    \centering
    \scalebox{0.9}{

\begin{tikzpicture}[
    ->, 
    >=stealth', 
    node distance=0.5\linewidth, 
    scale=0.8,
    every node/.style={scale=0.8}
    ]

  \node [state, initial] (q0) {$q0$};
  
  \node [state, below left of=q0] (q2) {$q2$};
  
  \node [state, below of=q0] (q3) {$q3$};
  
  \node [state, below right of=q0] (q4) {$q4$};
  
  \node [state, accepting, right of=q0, label=below right:1.00] (q1) {$q1$};

  \draw
    (q0) edge[loop above, above] node{\emptyAP:$0.56$, \carpet:$0.16$} (q0)
    (q0) edge[above] node{\green:$0.16$} (q1)

    (q0) edge[left] node{\water:$0.09$} (q2)

    (q2) edge[loop left, below] node{\water:$0.57$} (q2)
    (q2) edge[bend right, below] node{\emptyAP:$0.43$} (q3)
    
    (q3) edge[bend right, below] node{\emptyAP:$0.43$} (q4)
    
    (q3) edge[] node{\carpet:$0.43$} (q0)
    (q4) edge[right] node{\carpet:$0.67$} (q0);
\end{tikzpicture}}
    \subcaption{Pre-process, $\alpha=4$, Fig. \ref{fig:seshia}}
    \label{fig:pre}
\end{minipage}
\caption{The task specification and the learned PDFAs for the scenarios in Fig. \ref{fig:reachTwoGoals} and \ref{fig:seshia}. Each letter represents a symbol with a single atomic proposition s=$\{ship\}$, f=$\{fish\}$, b=$\{blue\}$, c=$\{carpet\}$, g=$\{green\}$, p=$\{purple\}$, and $e = \emptyset$. The termination probability $F_\PP$ of double-edged states is 1 and 0 at all other states.
The values in square brackets in (a) are the learned probabilities.
}
\label{fig:pdfas}
\end{figure*}

\begin{remark}
    Our algorithm can be used for static environments as well.
    Static environments can be viewed as dynamic environments, where the environment player is limited to one action at each state. 
\end{remark}

\subsection{\rev{Scalability Discussion}}


\rev{
As shown in Theorem~\ref{thm: pareto_points_computation}, the synthesis algorithm is polynomial in the size of the product game $\mathcal{G}^{\mathcal{P}}$, which is the Cartesian product of game $G$ and learned PDFA $\tilde{A}^{\mathbb{P}}$. The PDFA is typically small for robotic tasks, so the scalability of our framework is primarily influenced by the size of $G$, which depends on the problem-specific abstraction into a two-player game. While this abstraction process is well-studied and domain-dependent
(mobile robotics~\cite{Lahijanian:AR-CRAS:2018,Hadas:ICRA:2007,Lahijanian:ICRA:2009} and robotic manipulators~\cite{He:ICRA:2015,He:RAL:2019}), it can become computationally expensive, particularly in scenarios involving multiple environment agents.
}

\rev{
To mitigate this, symbolic representations such as Binary Decision Diagrams (BDDs) and Algebraic Decision Diagrams (ADDs) can be employed to compactly represent and manipulate large game graphs as shown in \cite{he2019efficient,muvvala2023efficient}. 
Furthermore, we highlight that our use of a learned PDFA for task representation (rather than an LTL-based specification), leads to significantly smaller automata and hence more efficient product construction. 
This is because the size of the DFA generated from co-safe LTL or LTL\textsubscript{f} task specifications is doubly exponential in the formula size. LTL specifications typically omit physical constraints, requiring the automaton to represent all logically possible executions. In contrast, our PDFA is learned from demonstrations that inherently reflect physical constraints, yielding a much smaller and more tractable automaton.}


\section{Case Studies and Evaluations}
\label{sec:caseStudy}




In this section, we evaluate the performance of the proposed algorithms across five case studies. We demonstrate that the solutions generated by our approach satisfy all the requirements outlined in \Cref{problem}. The case studies are designed to address the following key questions:






\begin{enumerate}[i.]
    \item Can the algorithm learn a non-Markovian task from demonstrations while capturing the demonstrator's preferences as a PDFA?
    \item Does the algorithm learn a PDFA that ensures safety is never violated, regardless of hyperparameters or the number of demonstrations?
    \item Do the synthesized strategies guarantee task completion in dynamic environments, while simultaneously maximizing preferences and minimizing robot cost?
    \item Is the algorithm applicable to real-world robotic systems?
\end{enumerate}

Our implementation of the EDSM algorithm is based on the MDI method that is used in the \textit{flexfringe} library \cite{verwer2017flexfringe}.
We call the basic algorithm the \textit{Vanilla} algorithm.
All the case studies were run on a MacBook Pro with 2.3 GHz Dual-Core Intel Core i5 and 16 GB RAM. Videos of all case studies are available to view \endnote{\url{https://youtu.be/TU8MhPBDBBs}}.

\subsection{Learning and Planning for Non-Markovian Tasks}
\label{subsec:caseStudy1}
In this case study, we consider the robotic scenario in \Cref{fig:reachTwoGoals}.  The task is to  visit both the school of fish and the shipwreck in any order and always avoid coral reefs. The preference is to visit  the shipwreck first.  A PDFA representation of this specification is shown in \Cref{fig:PDFA_twogoals}.

To learn this task, we sampled 1000 traces from this PDFA on the gridworld environment in \Cref{fig:reachTwoGoals}.  From these demonstrations, the Vanilla algorithm learned a PDFA with the same exact structure as the true PDFA and probabilities within 0.02 of the true values (in square brackets in Fig. \ref{fig:PDFA_twogoals}).

As the PDFA shows, our method correctly learned the non-Markovian task of visiting both the shipwreck and the school of fish in both orders and favors going to the shipwreck first.  Using this PDFA, our planner generated the robot trajectory shown in \Cref{fig:reachTwoGoals} (top), which correctly visits the shipwreck first and then the school of fish.
Next, we changed the environment by moving the location of the fish to be on the robot's way to the shipwreck as shown in \Cref{fig:reachTwoGoals} (bottom).
This figure also shows the synthesized plan in this environment using the same learned PDFA.  Notice that the robot does not visit the shipwreck first due to environmental constraint. Instead, it visits the fish and then the shipwreck, which is also a correct behavior.
This generality is the strength of learning the specification rather than learning a policy that is strongly dependent on the environment.

\subsection{Learning from Small Number of Samples with Safety}
\label{subsec:caseStudy2}

In this case study, we consider the environment and five demonstrations depicted in \Cref{fig:seshia} taken from \cite{vazquez2017learning} to learn the specification in a form of a PDFA as a comparison to the approach in \cite{vazquez2017learning}, which is based on learning specification formulas.
In this gridworld, each color represents an object, where orange is \emph{lava}, blue is \emph{water}, yellow is a  \emph{drying carpet}, white is an \emph{empty} space, and green is a \emph{charging station}.
The task is to reach a charging station.
However, the robot should not charge while it is wet. That is,
once it gets wet (goes to water), the robot has to dry at the \emph{drying carpet}.

\begin{figure*}
    \centering
    \begin{subfigure}[b]{0.32\textwidth}
        \centering
        \includegraphics[width=\textwidth]{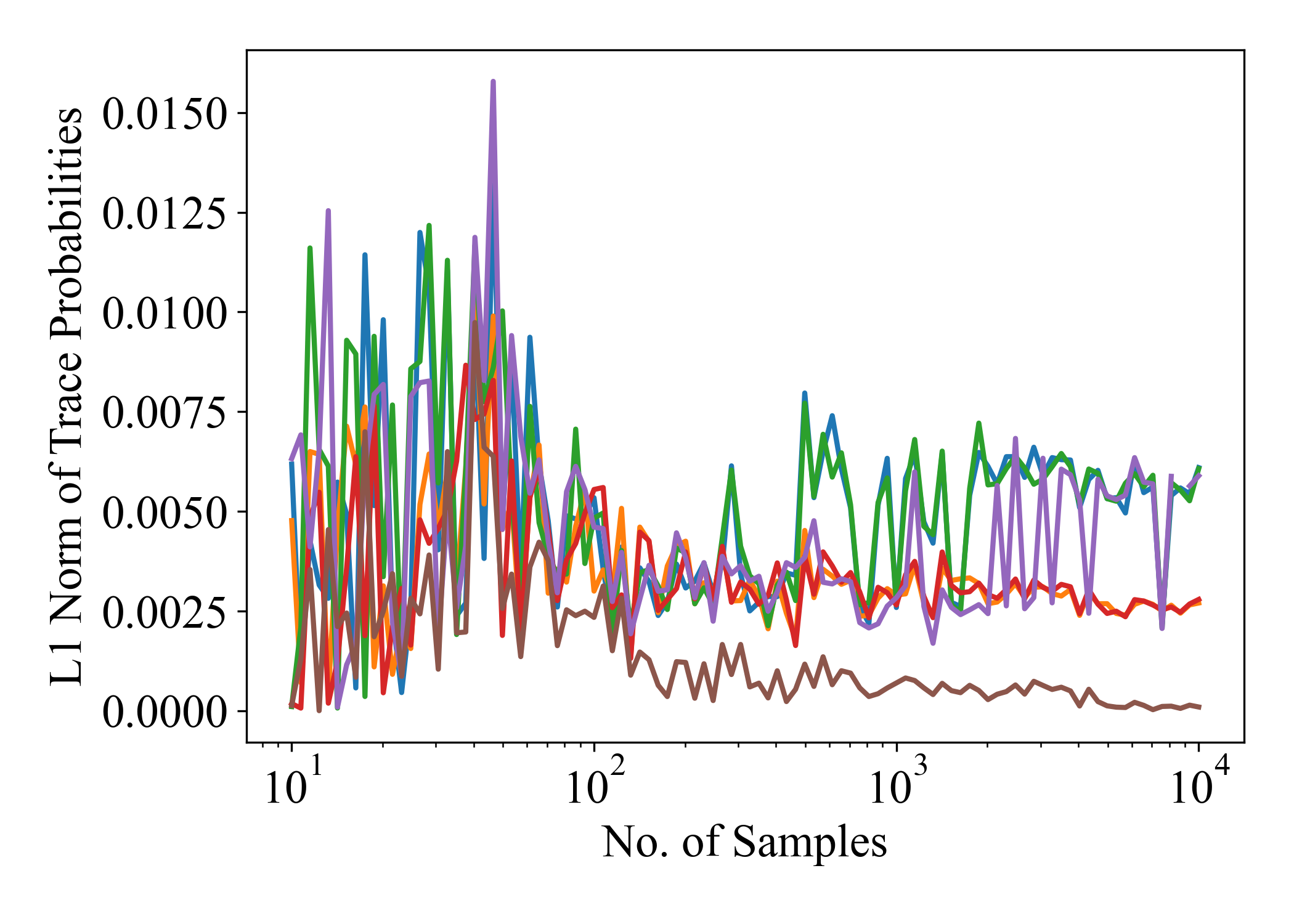} 
        \caption{L1 norm error (log scale)}
        \label{fig:kl}
    \end{subfigure}
    \hfill
    \begin{subfigure}[b]{0.325\textwidth}
        \centering
        \includegraphics[width=\textwidth]{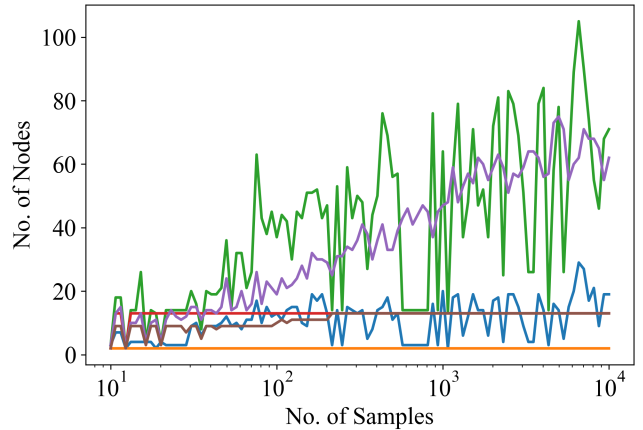}
        \caption{Number of nodes}
        \label{fig:nodes}
    \end{subfigure}
    \hfill
    \begin{subfigure}[b]{0.32\textwidth}
        \centering
        \includegraphics[width=\textwidth]{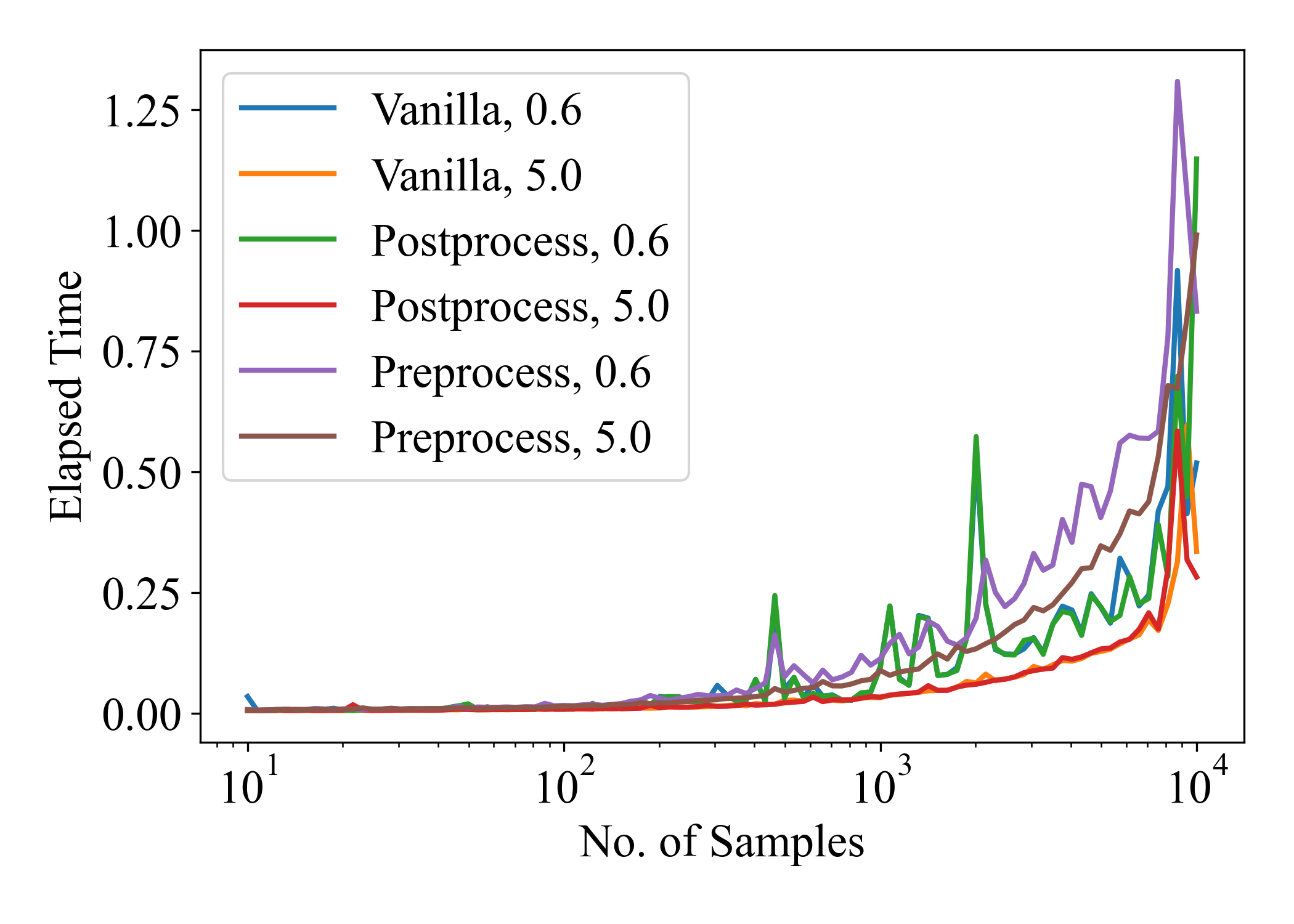}
        \caption{Computation time [s]}
        \label{fig:time}
    \end{subfigure}
    \caption{Performance analysis for the proposed algorithm. Plots in (a)-(b) use the same legend as (c).}
    \label{fig:results}
\end{figure*}

\subsubsection{Small number of samples}
We first used the Vanilla algorithm with the five demonstrations, which learned the PDFA in \Cref{fig:experiment1result}.
Note, in the learned PDFA, region green (\emph{charging station}) must always be observed to reach the final state. This shows that the task of reaching the \emph{charging station} is learned correctly. Next, on the right most branch of the PDFA, \emph{carpet} is always observed when the robot gets wet. Again, the algorithm succeeded in learning the task of visiting \emph{carpet} once the robots gets wet before reaching the \emph{charging station}.
One interesting observation is that the PDFA also learned that the robot has to go to the \emph{charging station} in one step after leaving the \emph{carpet}.  This is in fact a bias in the samples since every shown demonstration that includes \emph{carpet} has this property.  If that is the intention of the demonstrator, then it is a correct behavior.  If it is not, then it can be resolved by providing more samples.

Such one-step bias is not apparent in \cite{vazquez2017learning} because the ``next'' operator is not allowed in the syntax of the \textit{language} they consider. In contrary, our method infers over regular language, which includes the next operator.  Furthermore, in \cite{vazquez2017learning}, it took 95 seconds to learn the specification from just 5 demonstrations whereas ours took less than 0.01 seconds.

\subsubsection{Hyperparameter choice and safety}
The PDFA in \Cref{fig:experiment1result} is the result of the Vanilla algorithm when the hyperparameter of $\alpha$ is set to 0.4. It is a knob of how aggressive we allow the merges. Higher the value of $\alpha$ is, the smaller the PDFA becomes. If we can tune the hyperparameter correctly, we can get a desirable result as described above. But, if we increase $\alpha$ too much, some merges could induce unsafe behavior.   Unwanted merges occur because the algorithm is simply trying to \emph{minimize} the size of the structure.
In fact, the question of how to choose a correct value for $\alpha$ is an open problem.
For $\alpha = 4$, the learned PDFA from the same demonstrations is shown in \Cref{fig:spec}.  This PDFA has no regards for safety and only requires to reach the \emph{charging station}.
We can mitigate this problem by embedding safety specification. We define the following safety formula:
\begin{align*}
    \varphi_\safe = \mathcal{G} \neg \rm{lava} \wedge \mathcal{G} (\rm{water} \rightarrow \mathcal{X} (\varphi(\neg \rm{charge}, \rm{carpet}, k)),
\end{align*}
where
$\varphi$ is a formula recursively defined as: $\varphi (a, b, k) = a \land ( b \lor \mathcal{X} ( \varphi( a, b, k-1)))$ and $\varphi (a, b, 0) = a$,
and is read, ``visit $a$ for $k$ steps unless $b$ is visited''.
This formulas requires never going to \emph{lava} and, if the robot enters \emph{water}, it cannot \emph{charge} unless it visits \emph{carpet} or stays in \emph{empty} for $k$ consecutive steps to get dry. We set $k=10$ in all experiments.

From the same five demonstrations, we now learn PDFAs using the Post-process and Pre-process algorithms with $\alpha=4$ subject to $\varphi_\safe$.
The Post-process algorithm generates a large PDFA with 13 nodes and 36 edges since the safety DFA itself is large (12 nodes and 34 edges). Despite the size, it always guarantees no violation to $\varphi_\safe$.
The PDFA generated by the Pre-process algorithm is shown in \Cref{fig:pre}.  It is small and
correctly embeds both safety and liveness.
Further, all the demonstrations are accepted by both learned PDFA.
As for probabilities, the average L1 norm error was $1.65\times 10^{-3}$ for the Post-process PDFA and $7.42\times 10^{-5}$ for the Pre-process PDFA, indicating better performance by the Pre-process algorithm.
The larger error in the probabilities of the Post-process PDFA is due to the composition with the safety DFA, which prunes away the unsafe traces in the learned PDFA, corrupting the learned probability distributions.

\rev{
Furthermore, we note that learning the task as a PDFA enhances interpretability, providing both the designer and the demonstrator with an additional tool for tuning the hyperparameter $\alpha$, as illustrated in \Cref{fig:pdfas}. 
Next, we perform a thorough comparison of the learning methods by increasing the number of samples.
}

\subsection{Post-process versus Pre-process Algorithm}
\label{subsec:caseStudy3}

Here, the task is the same as the one above, but the goal is to quantitatively analyze and compare the performances of the proposed algorithms as the number of samples increases.
We sampled demonstrations randomly from the true PDFA
and used Post-process and Pre-process algorithms to learn PDFAs with hyperparameter values of $\alpha=0.6$ (less aggressive merge) and $\alpha=5.0$ (aggressive merge) to show the extreme results. We evaluated the resulting PDFAs with respect to the following metrics: L1 norm of the trace probability errors, number of states, and computation times.  

\rev{
Note that, a desirable method aims to reduce all three metrics.  That is, for PDFA learning, the smaller the number of states, the better it is as long as the automaton accurately represents the probability distributions over the accepting traces (language).  That means, the representation of the language (task) is compact.  This leads to several advantages, including faster strategy synthesis (because smaller PDFA results in smaller product game) and better interpretability.
}

\rev{
The results are shown in \Cref{fig:results} (all the plots share the same legend).
The results indicate that the Pre-process algorithm again performs better in both accuracy and size (but slower) than the others. From these results, we can say that the Pre-process algorithm is the best performing algorithm with respect to accuracy and automaton size. Moreover, its output PDFA does not violate the safety across all the trials (checked but not shown in the figures).
}

\begin{figure}
    \centering
    \begin{subfigure}[b]{0.35\linewidth}
        \centering
        \includegraphics[width=\linewidth]{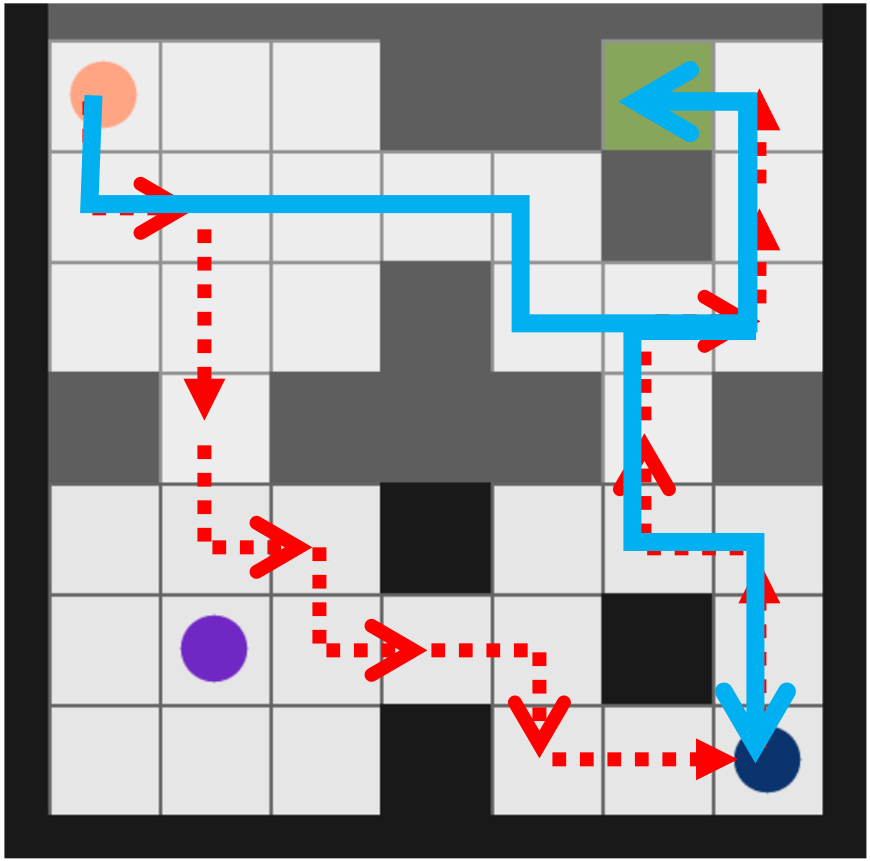}
        \caption{Fish and Shipwreck}
        \label{fig:dynamic_fish_and_shipwreck_env}
    \end{subfigure}
    ~~~
    \begin{subfigure}[b]{0.5\linewidth}
        \centering
        \includegraphics[width=\linewidth]{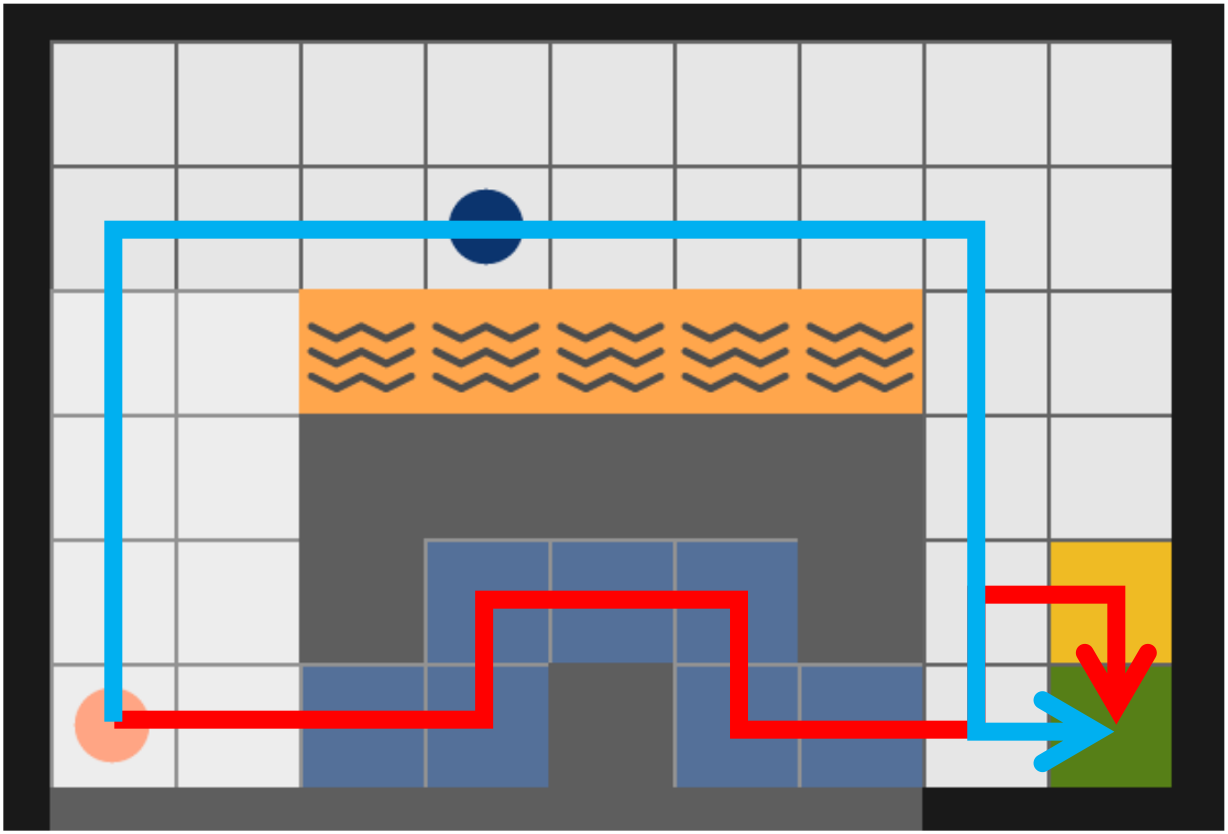}
        \caption{Charging Station}
        \label{fig:dynamic_charging_station_env}
    \end{subfigure}
    \caption{Dynamic MiniGrid Environments. Solid and dotted lines indicate 1-step and 2-step actions, respectively.}
    \label{fig:dynamic_minigrid}
\end{figure}

\subsection{Planning for Various Robots in different Environments}
\subsubsection{Planning in Static Environments}

From the learned PDFAs above, we picked one with a small L1 error norm.
Then, using this PDFA, we planned for various robots and environments  that are different from the one the demonstrations were shown in (see \Cref{fig:seshia}).
In all the cases, the computed plans correctly meet the requirements and preferences.
In the environment in \Cref{fig:gridworld2}, the \emph{lava} forces the robot to go to the bottom-right \emph{charging station}. Note that the robot avoids \emph{water} by going through \emph{carpet}, which is the preferred behavior.
In \Cref{fig:gridworld3}, we modified the robot's dynamics to only allow diagonal moves. The algorithm is again successful in generating a satisfying plan without violating the specification.
Because the specification is independent from any robotic systems and any environments, our framework is robust against the changes in the environment and robot dynamics.

\subsubsection{Planning in Dynamic Environments}
\label{sec: expriments dynamic env}
We synthesize a strategy for the learned PDFA in a dynamic environment. Let us recall the task of visiting both the school of fish and the shipwreck. We now assume the school of fish can dynamically move around freely. The new example is shown in \Cref{fig:dynamic_fish_and_shipwreck_env}. The red vehicle has to catch the blue fish and find the green shipwreck while avoiding the purple vehicle that can only move within the left bottom space. Moreover, we added more actions for the robot; it can choose to move one or two steps at a time. 
We set the energy cost of an action to be proportional to its number of steps. 
Taking two steps at a time will guarantee catching the fish in less number of steps but will cost more energy. 

Our algorithm found a Pareto front consisting of six Pareto points (cost and preference) that the robot can guarantee.
One of the points is $(40.0,40.72)$, i.e., maximum total payoff of $40.0$ for the energy cost and maximum preference ($-\mbox{log}\, P$) of $40.72$. A path obtained under the corresponding strategy is drawn in red in \Cref{fig:dynamic_fish_and_shipwreck_env}. Another Pareto point is $(59.0, 20.01)$, and a play under the corresponding strategy is drawn in blue in \Cref{fig:dynamic_fish_and_shipwreck_env}. 

We simulated 1000 plays on this game by choosing random strategies for the environment. All the obtained plays successfully completed the task and their total payoffs were bounded by the Pareto points. 
This case studies show that, regardless of the environment's behavior, the strategy computed by our algorithm can guarantee the completion of the task and maintain the total payoff within the chosen Parent point.


Here, we show that our algorithm can avoid another agent and complete the task in various ways. Recall the task of reaching the charging station from above.
We now consider the extended environment in \Cref{fig:dynamic_charging_station_env}.
Imagine the robot is an autonomous car and the blue agent is a pedestrian. 
The car has to avoid the pedestrian and reach the charging station or it can go through \emph{water} and get dried at \emph{carpet} to avoid the pedestrian.  We assume that the pedestrian can only walk around in the top two rows and the robot can take one or three steps to avoid conflicts with the pedestrian. 

Our algorithm found eight pareto points for this example. We show two distinct plays obtained by simulating the strategies in \Cref{fig:dynamic_charging_station_env}. The strategy that results in blue path 
corresponds to the Pareto point
$(47, 5.24)$, which trades off a more preferred way of achieving the task with a high robot cost. The red path strategy corresponds to the Pareto point $(12, 15.39)$, which guarantees lower robot cost but less preferred method of achieving the task. Paths corresponding to other Pareto points are similar to these two but their behavior changes based on the number of steps the robot takes per action. 

\begin{figure}[t]
    \centering
    \includegraphics[width=\linewidth]{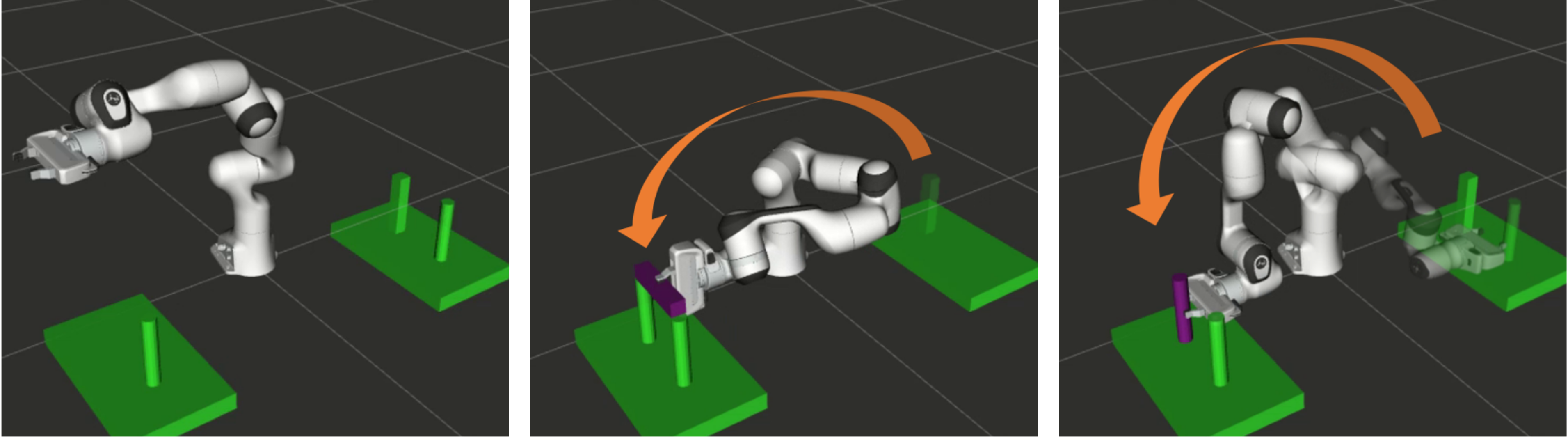}
    \caption{Manipulator completing the learned task of building an arch.}
    \label{fig:franka_sim}
\end{figure}
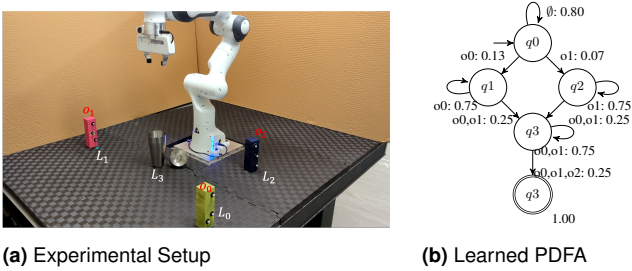
\begin{figure}[t]
    \centering
    \begin{subfigure}[b]{0.6\linewidth}
        \centering
        \includegraphics[width=1.0\linewidth]{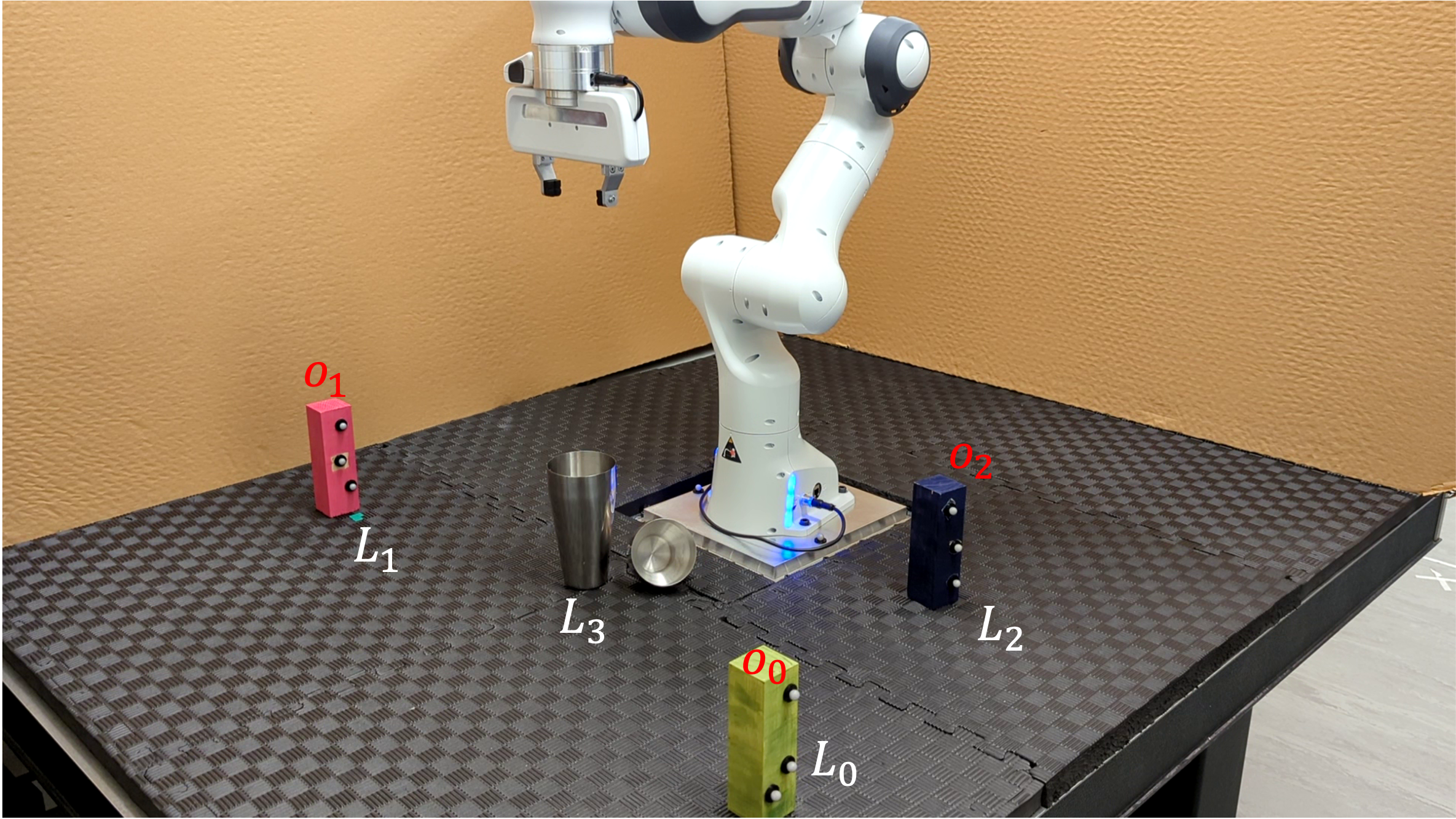}
        \caption{Experimental Setup}
        \label{fig:cocktail_experiment_setup}
    \end{subfigure}
    \hfill
    \begin{subfigure}[b]{0.35\linewidth}
        \centering
        \begin{minipage}[b]{1.0\linewidth}
            \centering
            \scalebox{.7}{

\begin{tikzpicture}[
    ->, 
    >=stealth', 
    node distance=0.5\linewidth, 
    scale=0.8,
    every node/.style={scale=0.8, font=\small},
    ]

  \node [state, initial] (q0) {$q0$};
  
  \node [state, below left of=q0] (q1) {$q1$};
  
  \node [state, below right of=q0] (q2) {$q2$};
  
  \node [state, below left of=q2] (q3) {$q3$};
  
  \node [state, accepting, below of=q3, label=below right:1.00] (q4) {$q3$};

  \draw
  
    (q0) edge[loop above, right, align=center] node[xshift=2mm, yshift=-2mm]{$\emptyset$: 0.80} (q0)
    
    (q1) edge[loop left, below, align=center] node[xshift=2mm, yshift=-2mm]{o0: 0.75} (q1)

    (q2) edge[loop right, below, align=center] node[xshift=-2mm, yshift=-2mm]{o1: 0.75} (q2)
    
    (q3) edge[loop right, below, align=center] node[xshift=-2mm, yshift=-2mm]{o0,o1: 0.75} (q3)
    
    (q0) edge[above left, align=center] node{o0: 0.13} (q1)
    (q0) edge[above right, align=center] node{o1: 0.07} (q2)
    (q1) edge[below left, align=center] node[xshift=2mm, yshift=0mm]{o0,o1: 0.25} (q3)
    (q2) edge[below right, align=center] node[xshift=2mm, yshift=0mm]{o0,o1: 0.25} (q3)
    
    (q3) edge[below right, align=center] node[xshift=-1mm, yshift=0mm]{o0,o1,o2: 0.25} (q4)
    ;
\end{tikzpicture}}
        \end{minipage}
        \caption{Learned PDFA}
        \label{fig:cocktail_making_pdfa}
    \end{subfigure}
    \caption{Cocktail Making Experiment}
    \label{fig:cocktail_example}
\end{figure}

\begin{figure*}
    \centering
    \begin{subfigure}[b]{\textwidth}
        \centering
        \includegraphics[width=\textwidth]{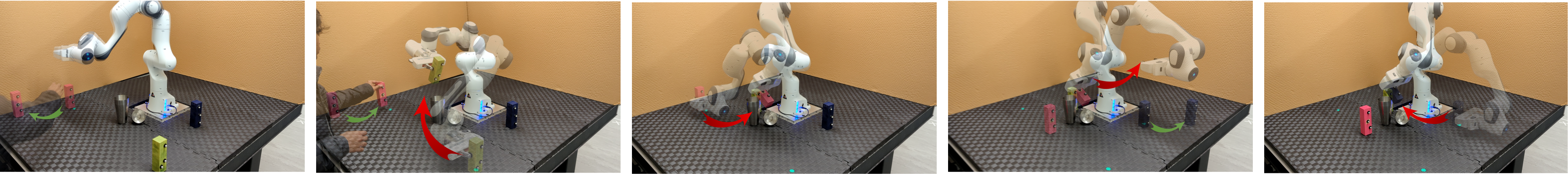}
        \caption{A play with the worst distance cost of [4.41, 7.88] by strategy $\strategy_0$}
        \label{fig:cocktail_experiments_worst_distance_cost}
    \end{subfigure}

    \begin{subfigure}[b]{\textwidth}
        \centering
        \includegraphics[width=\textwidth]{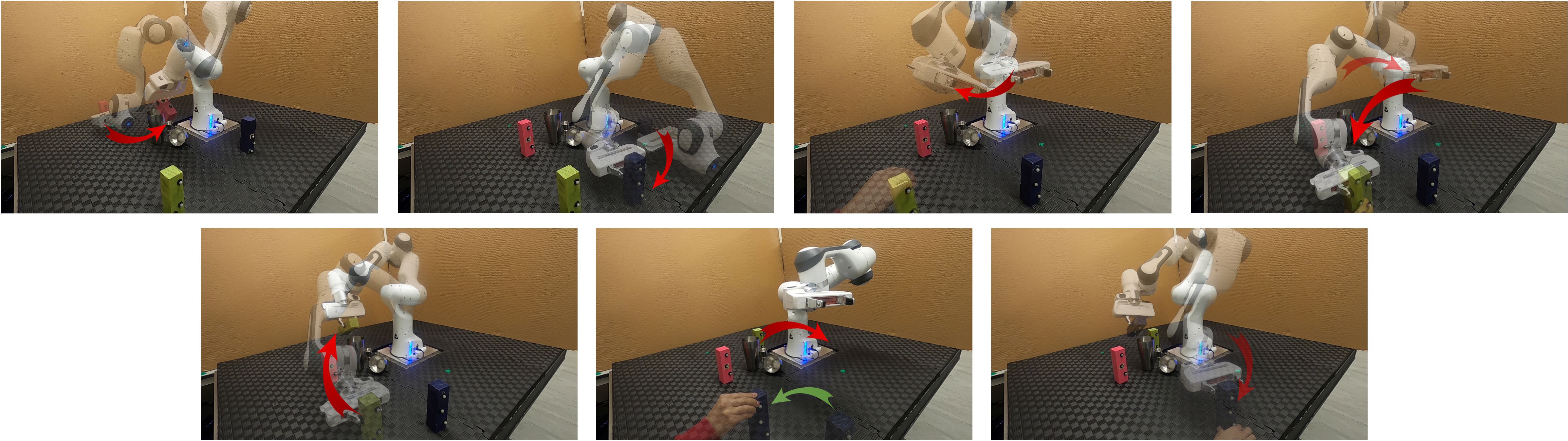}
        \caption{A play with the worst preference cost of [4.35, 12.65] by strategy $\strategy_0$}
        \label{fig:cocktail_experiments_worst_preference_cost}
    \end{subfigure}

    \caption{Plays of the strategy $\strategy_0$.}
    \label{fig:cocktail_experiments}
\end{figure*}

\subsection{Learning and Planning for Manipulation tasks}
\label{subsec:caseStudy4}
\subsubsection{Arch Building Example}

To show that our method is not limited to mobile robots, we considered a manipulation example in \Cref{fig:franka_sim} (left).
The robot is the Franka Emika Panda manipulator with 7 DoF, and the latent task is to build an arch with two cylindrical objects as columns and a rectangular box on top. The abstraction of the robot 
was done according to \cite{He:RAL:2019,He:ICRA:2015,muvvala2022regret}, which ended up with around 20,000 states.
The robot was given nine demonstrations: five most preferred (fastest), three mid-length (1.4 times as many actions), and one very bad demonstration (3 times as many actions).
A PDFA was learned with $\alpha=1.8$.  The learned PDFA has four states, and planning took 0.036 seconds.  The execution of the plan by the robot is shown in \Cref{fig:franka_sim} (middle and right), which shows that the robot successfully learned and executed the most preferred method of completing the task.

\subsubsection{Cocktail Making Example}
\label{sec: experiment cocktail}
To demonstrate the power of our reactive algorithm, we show a cocktail-making example in a human-robot collaboration setting. Imagine a robot and a human making cocktails individually next to each other in a bar kitchen.
The setup of the experiment is shown in \Cref{fig:cocktail_experiment_setup}. The blocks represent liquor bottles, which can be moved between predefined locations $L_0, L_1, L_2$, and $L_3$ by taking actions ``Transit-Grasp" and ``Transfer-Release". 
The human has two options: intervene and move an object at any time, or wait until the robot takes an action. However, there are certain restrictions. The human cannot intervene while the robot is holding an object, nor can they intervene on the same object twice. If the human does intervene on an object, that object must be returned within two robot action steps. The edge weights are automatically assigned based on the distance between each location.

The robot's task is to pour liquor $o_0, o_1$, and $o_2$ into a shaker while the human may intervene to borrow some of the liquors one at a time. The underlying task is to first pour $o_0$ and $o_1$ in any order and $o_2$ at the end. The most preferred way of picking objects is $o_0 \rightarrow o_1 \rightarrow o_2$. We sampled 3 demonstrations and learned a PDFA given a safety requirement of ``never observe $o_2$ before $o_0$ and $o_1$," i.e., $\mathcal{G} (o_2 \rightarrow X (\lnot o_0 \land \lnot o_1))$. The learned PDFA is presented in \Cref{fig:cocktail_making_pdfa}.

Our algorithm computed two distinct Pareto points, each representing a trade-off between distance cost and preference cost: $(4.41, 12.65)$ and $(4.68, 10.88)$. Additionally, it generated corresponding strategies, denoted as $\strategy_0$ and $\strategy_1$, for each Pareto point, respectively. 
These strategies produce plays with total payoffs bounded by their respective Pareto points.
For instance, strategy $\strategy_0$ pours the liquors in the order of $o_1$, $o_0$, and $o_2$, resulting in a total cost of $(3.44, 7.83)$. Conversely, strategy $\strategy_1$ follows the order $o_0$, $o_1$, and $o_2$, incurring a total cost of $(3.45, 7.21)$. Notably, since $o_1$ is closer (less distance cost) but less preferred than $o_0$, the total payoff of strategy $\strategy_0$ achieves a smaller distance cost at the expense of a higher preference cost compared to strategy $\strategy_1$.
Furthermore, due to the imposed safety requirement, the robot never attempts to pick up $o_2$ until the other objects have been retrieved.

In the case of human interventions, these strategies can still ensure that the worst-case total payoffs are within the bounds of the Pareto points. 
In \Cref{fig:cocktail_experiments}, we show the plays generated by strategy $\strategy_0$ that performed the worst total payoffs. 
Strategy $\strategy_1$ generated similar plays to those of strategy $\strategy_0$ with the flipped order of $o_0$ and $o_1$. 

In the scenario depicted in \Cref{fig:cocktail_experiments_worst_distance_cost}, the play started with the robot moving to location $L_1$ to retrieve object $o_1$. However, before the robot could grasp $o_1$, the human operator intervened. Due to the constraint that the human cannot intervene on another object while an intervention is already in progress, the robot deduced that no further intervention would occur. Consequently, it proceeded to pick up $o_0$ and pour its contents into the shaker at $L_3$. Upon the human returning the initially intervened object, the robot navigated back to pour $o_1$ into the shaker. Subsequently, when the robot attempted to pick up $o_2$, the human intervened again. The robot waited patiently until $o_2$ was returned, then poured the final liquor into the shaker.
In \Cref{fig:cocktail_experiments_worst_preference_cost}, the robot's strategy began with pouring $o_1$ into the shaker. Interestingly, it then transitioned to $o_2$, relocating it closer to $L_0$. While this action may seem surprising, it was a strategic move to minimize the cost of traveling between $L_0$ and $L_2$ in case of human intervention at $o_0$. Indeed, when the robot tried to pick up $o_0$, the human intervened, necessitating the robot's return to $L_2$ until $o_0$ was returned. Finally, the robot poured $o_0$ and $o_2$ into the shaker in order.

These case studies show that
our approach based on Pareto front computation allows for the generation of diverse strategies that cater to different priorities and constraints. By providing a range of Pareto optimal solutions, users can select the most suitable strategy based on the specific requirements of the application, such as prioritizing distance cost, preference cost, or striking a balance between the two.

\subsection{\rev{Benchmarks}}
\label{appendix: scalability}


\rev
{Here, we empirically assess the computational and scalability aspects of the proposed framework. To provide a thorough evaluation, we include all the case studies that involve dynamic environments, namely the ones in Sections~\ref{sec: expriments dynamic env} and~\ref{sec: experiment cocktail}, as well as an additional MiniGrid scenario involving two environment agents, illustrated in \Cref{fig: three-agents}. 
}
\rev{
In this scenario, the red robot is tasked with catching one of the blue (environment) agents and delivering it to the green region. Each blue agent has a rich action space, including movement in the four cardinal directions and diagonal moves. The red agent, by contrast, can only move in the four cardinal directions but with the advantage of taking three steps at a time, whereas the blue agents can take only one step at a time.}
\rev{
By taking the Cartesian product of the blue agents, we obtain a two-player game abstraction.  Our synthesis algorithm computes several Pareto-optimal winning strategies for the red agent.
}

\rev{
\Cref{table: scalability} presents a detailed breakdown of the computational performance of our framework across three representative MiniGrid environments as well as the manipulator (cocktail making). It reports 
the sizes (number of nodes and edges) of the learned PDFA $\PA$, game graph $\G$, and product game $\GPd$, as well as the computation times (in seconds) for the product game $\GPd = \PA \times \G$ construction, synthesis of the set of all the Pareto points (Pareto front) $\ppSet$, and synthesis of the set of Pareto optimal strategies $\mathrm{T}^* = \{\tau^*_\pp \mid \pp \in \ppSet \}$.
The table also includes the number of Pareto points $|\ppSet|$ for each experiment.  
}

\rev{
As expected, the size of the product game grows linearly with the size of the PDFA. Importantly, the Pareto-front synthesis algorithm, despite being the most computationally intensive step, remains tractable in practice, consistent with the polynomial-time complexity established in Theorem~\ref{thm: pareto_points_computation}. In fact, the strategy extraction times are even smaller.
}


\rev{
We note that the overall scalability of the framework is inherently limited by the size of the underlying game graph $\G$. 
The most expensive synthesis times in our experiments occur in the scenarios shown in \Cref{fig:dynamic_fish_and_shipwreck_env} and \Cref{fig: three-agents}, where the number of product game edges $|E^\mathcal{P}|$ is large, due to large number of actions in the underlying game graph $\G$.
This is a challenge acknowledged in the literature. To mitigate this, symbolic representations such as Binary Decision Diagrams (BDDs) or Algebraic Decision Diagrams (ADDs) have been proposed and could be incorporated to reduce memory usage and computational costs~\cite{he2019efficient,muvvala2023efficient}. 
%
%
}



\begin{table*}
    \centering
    \caption{
    \rev{Benchmark results across four dynamic environments. The table reports the sizes (number of nodes and edges) of the (i) learned PDFA $\PA$, (ii) game graph $\G$, and (iii) product game $\GPd$, 
    the computation times (in seconds) for the (iv) product game $\GPd = \PA \times \G$ construction, (v) synthesis of set of all the Pareto points (Pareto front) $\ppSet$, and (vi) synthesis of the set of Pareto optimal strategies $\mathrm{T}^* = \{\tau^*_\pp \mid \pp \in \ppSet \}$,
    and the number of Pareto points $|\ppSet|$.
    Notations $|E^{\PA}|$, $|E^{\G}|$, and $|E^{\mathcal{P}}|$ represent the number of edges of $\PA$, $\G$, and $\GPd$, respectively.
    }
    }
    \label{table: scalability}
    \scalebox{0.91}{
    \begin{tabular}{l l  c c  r r  r r r  r c c}
        \toprule
        \multirow{2}{*}{Scenario} & & \multicolumn{2}{c}{\underline{\hspace{7mm}$\PA$\hspace{6mm}}} & \multicolumn{2}{c}{\underline{\hspace{10mm}$\G$\hspace{10mm}}} & \multicolumn{3}{c}{\underline{\hspace{18mm}$\GPd$\hspace{18mm}}} 
        & \multicolumn{2}{c}{\underline{\hspace{3mm}Synthesis Time (s)\hspace{3mm}}} &
        \multirow{2}{*}{$|\ppSet|$}
        \\
        & & $|Q|$ & $|E^{\PA}|$ & \multicolumn{1}{c}{$|S|$} & \multicolumn{1}{c}{$|E^\G|$} & \multicolumn{1}{c}{$|S^{\mathcal{P}}|$} & \multicolumn{1}{c}{$|E^{\mathcal{P}}|$} & \multicolumn{1}{c}{Time (s)} &  
        \multicolumn{1}{c}{$\ppSet$} & \multicolumn{1}{c}{$\mathrm{T}^*$} &
        \\
        \midrule
        Fish \& Shipwreck & \Cref{fig:dynamic_fish_and_shipwreck_env} & 4 & 7 & 5328 & 202464  & 15626 & 585721 & 18.90 
        & 1087.13 & 91.02 & 1\\ 
        Charging Station & \Cref{fig:dynamic_charging_station_env} & 3 & 5 & 1886 & 61192 & 3245 & 100007 & 3.60  
        & 209.79 & 87.77 & 5\\ 
        Cocktail Making & \Cref{fig:cocktail_example} & 5 & 9 & 17662 & 39102 & 31765 & 67298  & 7.24 
        & 498.74 & 66.86 & 3 \\ 
        Three Agents & \Cref{fig: three-agents} & 4 & 7 & 13824 & 497664 & 39619 & 1376348 & 48.16 
        & 1626.92 & 278.50 & 2\\
        \bottomrule
    \end{tabular}
    }
\end{table*}

\rev{
Nonetheless, the reported results provide empirical evidence of the practical feasibility of our proposed framework for problems of moderate scale, especially in structured environments such as indoor robotics.
}

\begin{figure}
    \centering
    \includegraphics[width=0.5\linewidth]{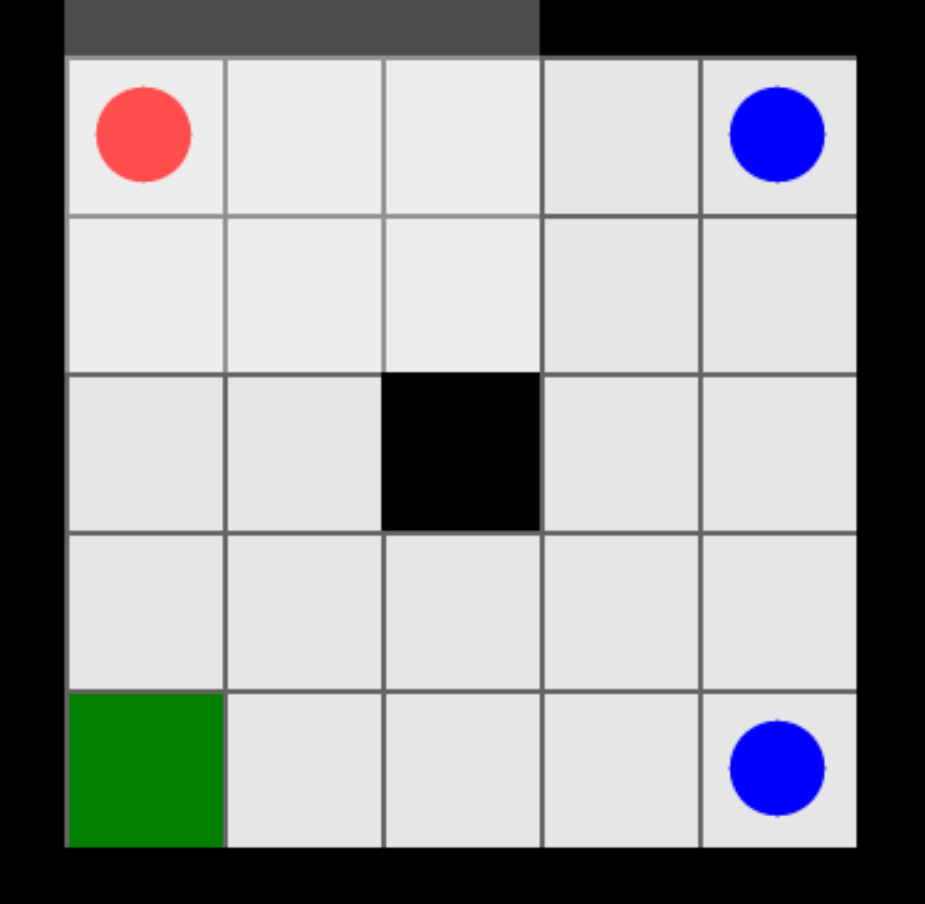}
    \caption{
    \rev{Red agent has to catch either of the blue agents and deliver it to the green region.  The action space of each blue agent consists of the four-cardinal and two-diagonal directions. The red agent can move in four cardinal directions and can take 3 steps at a time.}
    }
    \label{fig: three-agents}
\end{figure}

\section{Conclusion}
\label{sec:conclusion}
In this paper, we presented a new approach to learning specifications from demonstrations in the form of a PDFA. 
Unlike existing works, this method does not require prior knowledge, is fast, and captures preferences. 
We presented a pre-processing algorithm that incorporates safety constraints into the learning process. This algorithm significantly improved speed.
We also introduced a planning algorithm for learned specification while optimizing for multiple costs. 
The algorithm generates a set of all Pareto points that the user can choose from and a Pareto optimal strategy for each Pareto point.
Extensive evaluations illustrate the framework's flexibility and capability of robust knowledge transfer to various environments and robots.

Future directions for specification learning include inferring specifications over infinite horizons, embedding prior predictions or knowledge into the inference algorithm, and utilizing counterexamples to guide the inference. For strategy synthesis, our interest is on synthesizing strategies over infinite horizons, generating explanations for robotic behaviors, and recovering from unpredicted states (e.g., failures or unmodeled human interventions).


{\appendix

\section{Proof of Lemma~\ref{lemma: monotonic}}
\label{appendix: proof of lemma monotonic}
\begin{proof}
    This can easily be shown for the system nodes. As the algorithm progressively finds multiple feasible paths, it picks paths with the dominant total payoffs. Formally, by taking the union and the upper set of the total payoffs, only the dominant Pareto points remain in the set. Let $\tpSetAfter{i}(\GPdstate) = \{ u_1, \ldots, u_n  \}$ and assume a path with a smaller total payoff $u_1' \succeq u_1$ is found in one iteration. Then, 
    $$ \tpSetAfter{i+1}(\GPdstate) = \{ u_1',u_2, \ldots, u_n \} \succeq \{ u_1, u_2, \ldots, u_n \} = \tpSetAfter{i}(\GPdstate). $$
    
    At the environment nodes, the Pareto points at a state remain as infinities if the algorithm has not found a path to the terminating node from that state. The Pareto points only get updated after the Pareto points of its successor nodes get updated. Let $u_1$ be infinities, then
    $$ \tpSetAfter{i+1}(\GPdstate) =\{\vec{\infty} \} = \tpSetAfter{i}(\GPdstate). $$
    If all successor nodes have shorter paths, i.e., $u_i' \succeq u_i$ for all $i \in \{1, \ldots, n\}$, then 
    $$ \tpSetAfter{i+1}(\GPdstate) = \{ u_1', \ldots, u_n' \} \succeq \{ u_1, \ldots, u_n \} = \tpSetAfter{i}(\GPdstate). $$
    As the total payoffs at system nodes decrease monotonically, the set of all total payoffs at the environment nodes can only decrease monotonically. 
    Thus, we can derive $\tpSetAfter{i+1}(\GPdstate) \succeq \tpSetAfter{i}(\GPdstate)$ for all $\GPdstate$.
    Below, we show that this also holds when there exist strongly connected components (SCCs) in the game.
    
    Assume the game consists of SCCs.
    Environment nodes force a loop which leads to a non-winning region, but system nodes can break a loop if there exists a path to a terminating node. Let $\GPdstate_k$ be a node that has the option to exit the loop, $\GPdstate_{k+1}$ be its successor node that leads to the terminating node, and  $\play_{\rm{loop}} = \{\GPdstate_k, \GPdstate_{k'+1}, \ldots, \GPdstate_{k'+n}\}$ be a sequence of nodes in the loop.
    The smallest total payoffs at $\GPdstate_k$ are updated by taking the shortest path to the terminating node, i.e., 
    \rev{
    $\tpSetAfter{i}(\GPdstate_k) = \tpSetAfter{i-1}(\GPdstate_{k+1}) + \GPdWeight(\GPdstate_k, a, \GPdstate_{k+1})$, and the total payoffs at node $\GPdstate_{k'+1}$ is the sum of the total payoffs at $\GPdstate_k$ and the edge weights, i.e., $\tpSetAfter{i}(\GPdstate_k) \oplus \{\GPdWeight(s_{k'+n}, a, s_k) + \sum_{i=1}^{n-1}  \GPdWeight(s_{k'+i}, a, s_{k'+i+1})\}$, 
    where $\oplus$ is the Minkowski sum.}
    
    By taking the union of the total payoffs of its successor nodes, we get,
    \begin{align*} 
    \tpSetAfter{i+1}(\GPdstate_k) 
    & = \tpSetAfter{i}(\GPdstate_k) \cup \tpSetAfter{i}(\GPdstate_{k'+1}) \\
    & = \tpSetAfter{i}(\GPdstate_k) \cup \big(\tpSetAfter{i}(\GPdstate_k) \oplus \\ & \{ \GPdWeight(s_{k'+n}, a, s_k) + \sum_{i=1}^{n-1}  \GPdWeight(s_{k'+i}, a, s_{k'+i+1}) \}\big) \\
    & = \tpSetAfter{i}(\GPdstate_k)
    \end{align*}
    \rev{
    Since the sum of weights are all positive and the total payoffs are strictly greater than $\tpSet_i(\GPdstate_k)$, the union operation picks the dominant total payoffs $\tpSet_i(\GPdstate_k)$. Thus, the total payoff decreases monotonically even if there exists a loop.}
\end{proof}

\section{Proof of Proposition~\ref{lemma:maxpropagation}}
\label{appendix: proof of lemma:maxpropagation}

\begin{proof} The proof of Proposition~\ref{lemma:maxpropagation} relies on the following lemma.
    \begin{lemma}
    The maximum number of steps from the initial state $\GPdInit$ to the accepting state $\GPdTerm$ is $\GPdMaxStep$.
    \label{lemma:max_steps}
    \end{lemma}
    \begin{proof}
    If the initial state is in the winning region, there always exists a path from the initial state to the accepting state. Winning strategies take the shortest paths and never take a loop. This results in a visit at each node at most once. Therefore, the maximum number of steps (edges) the strategy can take is bounded by the number of nodes $\GPdMaxStep$.
    \end{proof}

    Applying the $\fpoperator_\ppSet$ at each node starting from the terminal state until the initial state explores all nodes and edges in the product game, hence $\GPdNumState + \GPdNumEdge$ at each iteration.
    From Lemma~\ref{lemma:max_steps}, every state must be reached in $\GPdMaxStep$ steps from the terminating state. Therefore, by reiterating the procedure $\GPdMaxStep$ times, all paths must be explored and the cost of all paths is taken into account.  
\end{proof}

\theendnotes

\bibliographystyle{SageH}
\bibliography{references}

@article{ravichandar2020recent,
  title={Recent advances in robot learning from demonstration},
  author={Ravichandar, Harish and Polydoros, Athanasios S and Chernova, Sonia and Billard, Aude},
  journal={Annual Review of Control, Robotics, and Autonomous Systems},
  volume={3},
  pages={297--330},
  year={2020},
  publisher={Annual Reviews}
}

@article{hussein2017imitation,
  title={Imitation learning: A survey of learning methods},
  author={Hussein, Ahmed and Gaber, Mohamed Medhat and Elyan, Eyad and Jayne, Chrisina},
  journal={ACM Computing Surveys (CSUR)},
  volume={50},
  number={2},
  pages={1--35},
  year={2017},
  publisher={ACM New York, NY, USA}
}

@book{sutton2018reinforcement,
  title={Reinforcement learning: An introduction},
  author={Sutton, Richard S and Barto, Andrew G},
  year={2018},
  publisher={MIT press}
}

@incollection{schaal2006dynamic,
  title={Dynamic movement primitives-a framework for motor control in humans and humanoid robotics},
  author={Schaal, Stefan},
  booktitle={Adaptive motion of animals and machines},
  pages={261--280},
  year={2006},
  publisher={Springer}
}

@article{paraschos2013probabilistic,
  title={Probabilistic movement primitives},
  author={Paraschos, Alexandros and Daniel, Christian and Peters, Jan and Neumann, Gerhard},
  journal={Neurips},
  year={2013}
}

@inproceedings{ng2000algorithms,
  title={Algorithms for inverse reinforcement learning.},
  author={Ng, Andrew Y and Russell, Stuart J},
  booktitle={{ICML}},
  volume={1},
  pages={2},
  year={2000}
}

@inproceedings{ziebart2008maximum,
  title={Maximum entropy inverse reinforcement learning.},
  author={Ziebart, Brian D and Maas, Andrew L and Bagnell, J Andrew and Dey, Anind K},
  booktitle={{AAAI}},
  volume={8},
  pages={1433--1438},
  year={2008},
  organization={Chicago, IL, USA}
}

@article{wulfmeier2015maximum,
  title={Maximum entropy deep inverse reinforcement learning},
  author={Wulfmeier, Markus and Ondruska, Peter and Posner, Ingmar},
  journal={arXiv:1507.04888},
  year={2015}
}

@inproceedings{ramachandran2007bayesian,
  title={Bayesian Inverse Reinforcement Learning.},
  author={Ramachandran, Deepak and Amir, Eyal},
  booktitle={IJCAI},
  volume={7},
  pages={2586--2591},
  year={2007}
}

@inproceedings{vazquez2017logical,
  title={Logical clustering and learning for time-series data},
  author={Vazquez-Chanlatte, Marcell and Deshmukh, Jyotirmoy V and Jin, Xiaoqing and Seshia, Sanjit A},
  booktitle={Computer Aided Verification},
  pages={305--325},
  year={2017},
  organization={Springer}
}

@inproceedings{li2017reinforcement,
  title={Reinforcement learning with temporal logic rewards},
  author={Li, Xiao and Vasile, Cristian-Ioan and Belta, Calin},
  booktitle={2017 IEEE/RSJ International Conference on Intelligent Robots and Systems},
  pages={3834--3839},
  year={2017},
  organization={IEEE}
}

@article{li2019formal,
  title={A formal methods approach to interpretable reinforcement learning for robotic planning},
  author={Li, Xiao and Serlin, Zachary and Yang, Guang and Belta, Calin},
  journal={Science Robotics},
  volume={4},
  number={37},
  year={2019},
  publisher={Science Robotics}
}

@inproceedings{ijcai2019-840,
  title     = {{LTL} and Beyond: Formal Languages for Reward Function Specification in Reinforcement Learning},
  author    = {Camacho, Alberto and Toro Icarte, Rodrigo and Klassen, Toryn Q. and Valenzano, Richard and McIlraith, Sheila A.},
  booktitle = {Int'l Joint Conference on Artificial Intelligence},
  pages     = {6065--6073},
  year      = {2019},
  month     = {7},
}

@article{xu2018advisory,
  title={Advisory temporal logic inference and controller design for semiautonomous robots},
  author={Xu, Zhe and Saha, Sayan and Hu, Botao and Mishra, Sandipan and Julius, A Agung},
  journal={IEEE Transactions on Automation Science and Engineering},
  volume={16},
  number={1},
  pages={459--477},
  year={2018},
  publisher={IEEE}
}

@inproceedings{jha2017telex,
  title={Telex: Passive {stl} learning using only positive examples},
  author={Jha, Susmit and Tiwari, Ashish and Seshia, Sanjit A and Sahai, Tuhin and Shankar, Natarajan},
  booktitle={International Conference on Runtime Verification},
  pages={208--224},
  year={2017},
  organization={Springer}
}

@inproceedings{shah2018bayesian,
  title={Bayesian inference of temporal task specifications from demonstrations},
  author={Shah, Ankit Jayesh and Kamath, Pritish and Li, Shen and Shah, Julie A},
  page={},
  year={2018},
  booktitle={Neural Information Processing Systems Foundation, Inc.}
}

@inproceedings{vazquez2017learning,
  title={Learning Task Specifications from Demonstrations},
  author={Vazquez-Chanlatte, Marcell and Jha, Susmit and Tiwari, Ashish and Ho, Mark K and Seshia, Sanjit},
  booktitle={NeurIPS},
  volume={31},
  year={2018}
}

@inproceedings{araki2019learning,
  title={Learning to plan with logical automata},
  author={Araki, Brandon and Vodrahalli, Kiran and Leech, Thomas and Vasile, Cristian-Ioan and Donahue, Mark D and Rus, Daniela L},
  year={2019},
  booktitle={Robotics: Science and Systems Foundation}
}

@inproceedings{verwer2017flexfringe,
  title={Flexfringe: a passive automaton learning package},
  author={Verwer, Sicco and Hammerschmidt, Christian A},
  booktitle={Intl. Conf. Software Maintenance and Evolution (ICSME)},
  pages={638--642},
  year={2017},
  organization={IEEE}
}

@book{de2010grammatical,
  title={Grammatical inference: learning automata and grammars},
  author={De la Higuera, Colin},
  year={2010},
  publisher={Cambridge University Press}
}

@inproceedings{he2017reactive,
  title={Reactive synthesis for finite tasks under resource constraints},
  author={He, Keliang and Lahijanian, Morteza and Kavraki, Lydia E and Vardi, Moshe Y},
  booktitle={2017 IEEE/RSJ International Conference on Intelligent Robots and Systems (IROS)},
  pages={5326--5332},
  year={2017},
  organization={IEEE}
}

@inproceedings{chen2013synthesis,
  title={Synthesis for multi-objective stochastic games: An application to autonomous urban driving},
  author={Chen, Taolue and Kwiatkowska, Marta and Simaitis, Aistis and Wiltsche, Clemens},
  booktitle={International Conference on Quantitative Evaluation of Systems},
  pages={322--337},
  year={2013},
  organization={Springer}
}

@inproceedings{basset2015strategy,
  title={Strategy synthesis for stochastic games with multiple long-run objectives},
  author={Basset, Nicolas and Kwiatkowska, Marta and Topcu, Ufuk and Wiltsche, Clemens},
  booktitle={International Conference on Tools and Algorithms for the Construction and Analysis of Systems},
  pages={256--271},
  year={2015},
  organization={Springer},
}

@inproceedings{chen2013stochastic,
  title={On stochastic games with multiple objectives},
  author={Chen, Taolue and Forejt, Vojt{\v{e}}ch and Kwiatkowska, Marta and Simaitis, Aistis and Wiltsche, Clemens},
  booktitle={International Symposium on Mathematical Foundations of Computer Science},
  pages={266--277},
  year={2013},
  organization={Springer},
}

@inproceedings{chatterjee2012strategy,
  title={Strategy synthesis for multi-dimensional quantitative objectives},
  author={Chatterjee, Krishnendu and Randour, Mickael and Raskin, Jean-Fran{\c{c}}ois},
  booktitle={International Conference on Concurrency Theory},
  pages={115--131},
  year={2012},
  organization={Springer}
}

@article{sastry2005new,
  title={New polynomial time algorithms to compute a set of Pareto optimal paths for multi-objective shortest path problems},
  author={Sastry, VN and Janakiraman, TN and Mohideen, S Ismail},
  journal={International Journal of Computer Mathematics},
  volume={82},
  number={3},
  pages={289--300},
  year={2005},
  publisher={Taylor \& Francis}
}

@inproceedings{fainekos2005temporal,
  title={Temporal logic motion planning for mobile robots},
  author={Fainekos, Georgios E and Kress-Gazit, Hadas and Pappas, George J},
  booktitle={Proceedings of the 2005 IEEE International Conference on Robotics and Automation},
  pages={2020--2025},
  year={2005},
  organization={IEEE}
}

@inproceedings{bhatia2010sampling,
  title={Sampling-based motion planning with temporal goals},
  author={Bhatia, Amit and Kavraki, Lydia E and Vardi, Moshe Y},
  booktitle={2010 IEEE International Conference on Robotics and Automation},
  pages={2689--2696},
  year={2010},
  organization={IEEE}
}

@inproceedings{camacho2019ltl,
  title={LTL and Beyond: Formal Languages for Reward Function Specification in Reinforcement Learning.},
  author={Camacho, Alberto and Icarte, Rodrigo Toro and Klassen, Toryn Q and Valenzano, Richard Anthony and McIlraith, Sheila A},
  booktitle={IJCAI},
  volume={19},
  pages={6065--6073},
  year={2019}
}

@inproceedings{puri2013efficient,
  title={Efficient parallel and distributed algorithms for GIS polygonal overlay processing},
  author={Puri, Satish and Prasad, Sushil K},
  booktitle={2013 IEEE International Symposium on Parallel \& Distributed Processing, Workshops and Phd Forum},
  pages={2238--2241},
  year={2013},
  organization={IEEE}
}

@InProceedings{Hadas:ICRA:2007,
  author =       {H. Kress-Gazit and G. Fainekos and G. J. Pappas},
  title =        {Where's {W}aldo? Sensor-based temporal logic motion planning},
  booktitle =    {Int. Conf. on Robotics and Automation},
  year =         {2007},
  pages=         {3116--3121},
  address =      {Rome, Italy},
  organization={IEEE},
}

@INPROCEEDINGS{Lahijanian:ICRA:2009,
  AUTHOR =       {Morteza Lahijanian and M. Kloetzer and S. Itani and C. Belta and S.B. Andersson},
  TITLE =        {Automatic deployment of autonomous cars in a robotic urban-like environment {(RULE)}},
  BOOKTITLE =    {Int. Conf. on Robotics and Automation},
  YEAR =         {2009},
  pages =        {2055--2060},
  address =      {Kobe, Japan},
  organization = {IEEE},
}

@BOOK{BaierBook2008,
  AUTHOR =       {Christel Baier and Joost-Pieter Katoen},
  TITLE =        {Principles of Model Checking},
  PUBLISHER =    {The MIT Press},
  YEAR =         {2008},
  address =      {Cambridge, MA},
}

@article{kupferman:FMSD:2001,
  author = {Orna Kupferman and Moshe Y. Vardi},
  issue = {3},
  journal = {Formal Methods in System Design},
  pages = {291--314},
  title = {Model Checking of Safety Properties},
  volume = {19},
  year = {2001}
}

@InProceedings{He:ICRA:2015,
  title = {Towards Manipulation Planning with Temporal Logic Specifications},
  author = {Keliang He and Morteza Lahijanian and Lydia E. Kavraki and Moshe Y. Vardi},
  booktitle = {Int. Conf. Robotics and Automation},
  year = {2015},
  month = {May},
  pages = {346--352},
  publisher = {IEEE}
}

@ARTICLE{Lahijanian:TRO:2016,
  AUTHOR =       {Morteza Lahijanian and Matthew R. Maly and Dror Fried and Lydia E. Kavraki and Hadas Kress-Gazit and Moshe Y. Vardi},
  TITLE =        {Iterative Temporal Planning in Uncertain Environments With Partial Satisfaction Guarantees},
  JOURNAL =      {IEEE Transactions on Robotics},
  YEAR =         {2016},
  volume =       {32},
  number =       {3},
  pages =        {538--599},
  month =        {May},
  publisher =    {IEEE},
  doi =          {10.1109/TRO.2016.2544339},
}

@INPROCEEDINGS{He:IROS:2017,
  AUTHOR =       {Keling He and Morteza Lahijanian and Lydia E. Kavraki and Moshe Y. Vardi},
  TITLE =        {Reactive Synthesis for Finite Tasks Under Resource Constraints},
  BOOKTITLE =    {Int. Conf. on Intelligent Robots and Systems (IROS)},
  month =        {Sep.},
  YEAR =         {2017},
  publisher =    {IEEE},
  pages =        {5326--5332},
  address =      {Vancouver, BC, Canada},
}

@ARTICLE{Lahijanian:AR-CRAS:2018,
  AUTHOR =       {Hadas Kress-Gazit and Morteza Lahijanian and Vasumathi Raman},
  TITLE =        {Synthesis for Robots: Guarantees and Feedback for Robot Behavior},
  JOURNAL =      {Annual Review of Control, Robotics, and Autonomous Systems},
  YEAR =         {2018},
  MONTH =        {May},
  VOLUME =       {1},
  PAGES =        {211--236},
  DOI =          {10.1146/annurev-control-060117-104838}
}

@article{He:RAL:2019,
  title={Automated Abstraction of Manipulation Domains for Cost-Based Reactive Synthesis},
  author={He, Keliang and Lahijanian, Morteza and Kavraki, Lydia, E and Vardi, Moshe, Y},
  journal={IEEE Robotics and Automation Letters},
  volume={4},
  number={2},
  pages={285--292},
  year={2019},
  month = {Apr.},
  publisher={IEEE},
}

@INPROCEEDINGS{muvvala2022regret,
  author={Muvvala, Karan and Amorese, Peter and Lahijanian, Morteza},
  booktitle={2022 International Conference on Robotics and Automation (ICRA)}, 
  title={Let's Collaborate: Regret-based Reactive Synthesis for Robotic Manipulation}, 
  year={2022},
  pages={4340-4346}}

@inproceedings{watanabe2021probabilistic,
  title={Probabilistic specification learning for planning with safety constraints},
  author={Watanabe, Kandai and Renninger, Nicholas and Sankaranarayanan, Sriram and Lahijanian, Morteza},
  booktitle={2021 IEEE/RSJ International Conference on Intelligent Robots and Systems (IROS)},
  pages={6558--6565},
  year={2021},
  organization={IEEE}
}

@inproceedings{muvvala2023efficient,
  title={Efficient symbolic approaches for quantitative reactive synthesis with finite tasks},
  author={Muvvala, Karan and Lahijanian, Morteza},
  booktitle={2023 IEEE/RSJ International Conference on Intelligent Robots and Systems (IROS)},
  pages={8666--8672},
  year={2023},
  organization={IEEE}
}

@inproceedings{he2019efficient,
  title={Efficient symbolic reactive synthesis for finite-horizon tasks},
  author={He, Keliang and Wells, Andrew M and Kavraki, Lydia E and Vardi, Moshe Y},
  booktitle={2019 International Conference on Robotics and Automation (ICRA)},
  pages={8993--8999},
  year={2019},
  organization={IEEE}
}

@inproceedings{Muvvala:ICRA:2024,
  title = {Stochastic Games for Interactive Manipulation Domains},
  author = {Muvvala, Karan and Wells, Andrew and Lahijanian, Morteza and Kavraki, Lydia and Vardi, Moshe},
  booktitle = {2024 IEEE Conference on Robotics and Automation (ICRA)},
  month = may,
  year = {2024},
  address = {Yokohama, Japan},
  organization = {IEEE},
  bibkey = {maniprs},
  videourl = {https://www.youtube.com/watch?v=OTosSD7yPA0},
  url = {https://arxiv.org/abs/2403.04910},
  doi = {10.1109/ICRA57147.2024.10611623}
}

\end{document}